\documentclass{article}

\PassOptionsToPackage{numbers, compress}{natbib}
\usepackage[preprint]{neurips_2026}

\usepackage[utf8]{inputenc} % allow utf-8 input
\usepackage[T1]{fontenc}    % use 8-bit T1 fonts
\usepackage{hyperref}       % hyperlinks
\usepackage{url}            % simple URL typesetting
\usepackage{booktabs}       % professional-quality tables
\usepackage{amsfonts}       % blackboard math symbols
\usepackage{nicefrac}       % compact symbols for 1/2, etc.
\usepackage{microtype}      % microtypography
\usepackage{xcolor}         % colors

\usepackage{amsmath}
\usepackage{amssymb}
\usepackage{amsthm}
\usepackage{mathtools}
\usepackage{microtype}
\usepackage{graphicx}
\usepackage{booktabs}
\usepackage{enumitem}
\usepackage{wrapfig}
\usepackage{afterpage}
\usepackage[table]{xcolor}
\usepackage{tikz-cd}

\theoremstyle{plain}  \newtheorem{proposition}{Proposition}
\theoremstyle{remark} \newtheorem{remark}{Remark}
\theoremstyle{corollary} \newtheorem{corollary}{Corollary}
\theoremstyle{theorem} \newtheorem{theorem}{Theorem}
\newtheorem{assumption}{Assumption}
\newtheorem{lemma}{Lemma}
\newcommand{\R}{\mathbb{R}}
\title{A meshfree exterior calculus for generalizable and data-efficient learning of physics from point clouds}

\author{%
  Benjamin D.~Shaffer\thanks{Mechanical Engineering and Applied Mechanics}%\thanks{Use footnote for providing further information
    \\
  University of Pennsylvania\\
  Philadelphia, PA 19104 \\
  \texttt{ben31@seas.upenn.edu} \\
  \And
  Brooks Kinch$^*$\\
  University of Pennsylvania\\
  Philadelphia, PA 19104 \\
  \And
  M. Ani Hsieh$^*$\\
  University of Pennsylvania\\
  Philadelphia, PA 19104 \\
  \And
  Nathaniel Trask$^*$\\
  University of Pennsylvania\\
  Philadelphia, PA 19104 \\
  \texttt{ntrask@seas.upenn.edu}\\
}

\begin{document}

\maketitle

\begin{abstract}
We introduce a meshfree exterior calculus (MEEC) for learning structure-preserving descriptions of physics on point clouds, and use it to build MEEC-Net, a data-efficient surrogate that transfers across resolutions, geometries, and physical parameters. MEEC equips an $\epsilon$-ball graph with virtual node and edge measures via a single sparse Schur complement solve; the resulting complex satisfies discrete conservation exactly, is end-to-end differentiable in the point positions, and exposes a direct geometry-to-physics link without the mesh-generation step required by conventional structure-preserving discretizations. MEEC-Net learns unknown physics as a shared edge-wise flux law in an SO($d$)-invariant local frame, so the same kernel produces compatible fluxes on any point cloud whose features lie in the training range. We prove a solution-error bound that splits into discretization and kernel-approximation terms which is independent of problem geometry, explaining the observed transfer from very few examples. We show that single-solution training transfers to unseen geometries, boundary conditions, and physical parameters. On five canonical PDE benchmarks MEEC-Net achieves 1–2 orders of magnitude lower out-of-distribution error than baseline neural-operator approaches. On the SimJEB structural-bracket benchmark it achieves competitive error while using substantially fewer training geometries.
% Code is available at \TODO{insert public repo URL}.
% and without synth augmentation-based reported results.
\end{abstract}

% ============================================================
\section{Introduction}
\label{sec:intro}
\begin{wrapfigure}{r}{0.4\textwidth}
    \vspace{-\intextsep}
    \centering
    \includegraphics[width=\linewidth]{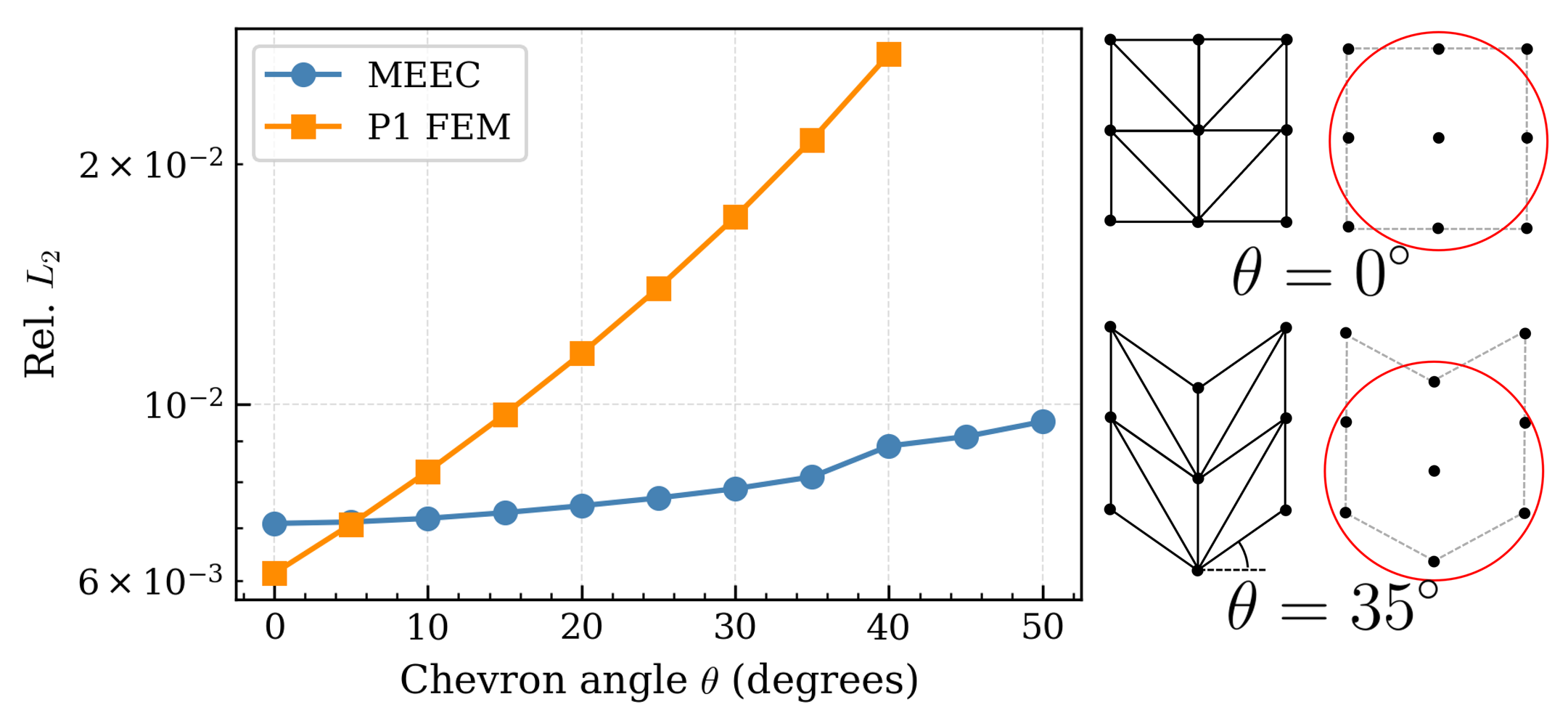}
    \caption{Mesh-based simulation is sensitive to mesh quality: on a chevron mesh with skew $\theta$, FEM fails to converge at $\theta>40^\circ$ while MEEC remains accurate. Meshfree stencils (red) trade compact stencils for robustness.}\label{fig:chevron_mesh}
\end{wrapfigure}
Direct neural operators learn global solution maps from examples; when the shape, boundary conditions, resolution, or physical parameters shift, they must extrapolate how the solution field changes from examples in the training distribution \citep{lu2021learning, li2020neural}.
Physical systems, however, are governed by a local constitutive law whose form is shared across problem specifications.
This suggests a different learning approach where the learned object is instead a \textit{local} flux law, and a point-cloud-native discretization assembles and solves the \textit{global} PDE on each new geometry.

Point clouds are a natural representation of geometry across sensing, simulation, and learning pipelines \citep{qi2017pointnet, kashefi2021point, kerbl20233d}. Meshfree PDE discretizations on point clouds are a niche but mature field \citep{belytschko1996meshless, liu2009meshfree} and avoid the brittle mesh-generation step of conventional methods (Figure~\ref{fig:chevron_mesh}) \citep{boggs2005dart}. However, they retain only point-evaluation degrees of freedom (DOFs), lacking a natural mechanism (e.g. volume/area DOFs) to preserve topological structure at the discrete level. Discrete exterior calculus (DEC) \citep{hirani2003discrete} and its finite-element counterpart (FEEC) \citep{arnold2018finite} provide powerful algebraic formalisms to describe topological structure, but both require an underlying mesh, as do recent data-driven extensions \citep{trask2022enforcing, actor2024data, kinch2025structure, shaffer2026structure}. We address this gap by introducing a \textbf{ME}shfree \textbf{E}xterior \textbf{C}alculus (MEEC).

\begin{figure}[!t]
    \centering
    \includegraphics[width=0.83\linewidth]{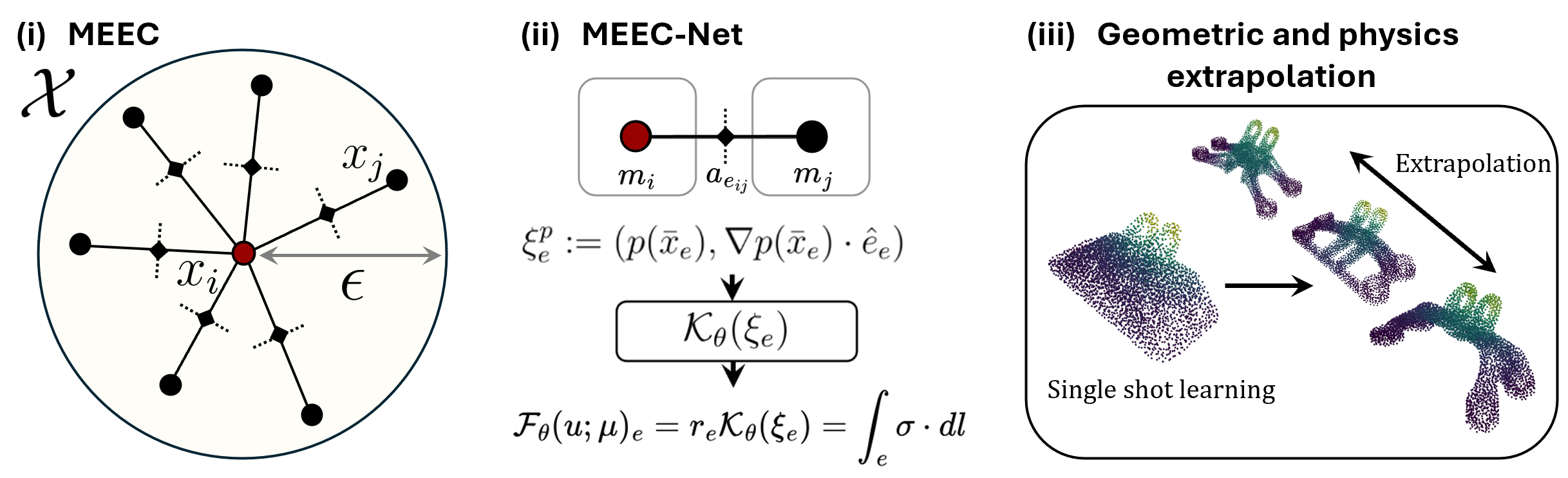}
    \caption{\textbf{Paper overview.} \textit{Left:} MEEC attaches virtual volumes and areas to an $\epsilon$-ball graph via a quadratic program, obtaining an end-to-end differentiable, convergent discretization of conservation laws natively on a point cloud. \textit{Center:} MEEC-Net learns an $SO(d)$-invariant description of physics mapping state $(u_i,u_j)$ to flux $\int_{e_{ij}}\sigma \cdot dl$, able to generalize across edge lengths and orientations unobserved in training. \textit{Right:} From a single observation, we are able to extract physical models that extrapolate across geometries and physics while preserving physical structure, outperforming SotA neural operators trained with dramatically more data.}
    \label{fig:placeholder_overview_fig}
\end{figure}
MEEC is built on the $\epsilon$-ball graph of point-cloud nodes in $\mathbb{R}^d$: a coboundary $d_0$ exact in the cochain sense (by the fundamental theorem of calculus, $(d_0 u)_e = \int_e \nabla u\cdot d\ell = u_+ - u_-$) and an adjoint divergence operator constructed to preserve polynomial consistency, obtained via a single sparse linear solve. The resulting discrete de~Rham complex is end-to-end differentiable in the point positions. On this complex, MEEC-Net then learns a state-to-flux constitutive law from frame-invariant edge features, trained from solution data by implicit differentiation through the converged PDE solve. MEEC handles geometry, resolution, boundary conditions, and global coupling while MEEC-Net learns only a reusable local flux law; this separation enables the data-efficient generalization demonstrated in Section~\ref{sec:results}.

\textit{Contributions:}

\textbf{1. MEEC} (\S\ref{sec:mfdec}). We present MEEC for PDEs on point clouds~\eqref{eq:complex}. We construct a point-cloud discretization comprising nodal cochains, oriented edge cochains, an incidence coboundary, and a moment-matched Hodge star obtained from a single sparse Schur solve; the operator is conservative by algebraic identity and $O(h)$-consistent, without any mesh or dual-cell construction. 

\textbf{2. MEEC-Net: local constitutive learning on edge cochains.} (\S\ref{sec:learned}).
On the MEEC complex, we parameterize unknown physics as a shared edge-local flux law on frame-invariant edge features of the state and physical parameters. The learned flux is assembled as a consistent 1-cochain and solved inside the MEEC conservation law, so the network learns a local constitutive relation rather than a global solution map. Frame-invariance and a prescription of fluxes as integrands of edge integrals enable extrapolation.

\textbf{3. Error decomposition for learned PDEs.} (\S\ref{sec:theory}). We prove a consistency-plus-stability estimate showing that solution error separates into MEEC discretization error, stability constants, and local flux-kernel approximation error. Geometry and boundary conditions enter only through stability and discretization constants that are uniform over a fixed point-cloud regularity class and edge-feature range, and the learned kernel error depends only on approximation over that feature range.

\textbf{4. Low-data generalization across point clouds and geometries.} (\S\ref{sec:results}). On canonical PDE benchmarks and an engineering-relevant jet engine bracket design problem, MEEC-Net achieves state-of-the-art transfer across resolutions, geometries, boundary conditions, and physical parameters, particularly in low-data regimes where direct neural surrogates must infer a global solution operator from few examples.

\section{Related Work}

\paragraph{Machine learning for PDEs.}
Machine learning is used in PDE problems for surrogate modeling, inverse problems, closure modeling, and equation discovery~\citep{karniadakis2021physics, sanderse2024scientific}. Physics-informed neural networks represent the solution with a neural network from a known residual~\citep{raissi2017physics}. Neural operators learn data-driven maps between function spaces, typically from coefficients, forcing, boundary data, or geometry to the solution field~\citep{lu2021learning,kovachki2023neural}. Recent operator architectures use spectral layers~\citep{kovachki2023neural}, graph message passing~\citep{pfaff2020learning,brandstetter2022message, zeng2024phympgn}, attention mechanisms~\citep{calvello2025continuum,wu2024transolver,alkin2024universal,hao2023gnot, liu2026geometry, wang2024latent}, and geometry-aware representations for irregular domains~\citep{li2023geometry,wen2025geometry,franco2023mesh,yamazaki2025finite,bleeker2025neuralcfd,serrano2024aroma,alkin2025ab}. These methods produce effective surrogates, but their global solution maps limit extrapolation beyond training data.
A alternative line of work defines predictions implicitly, as the solution of an optimization problem or fixed-point equation rather than as a direct network output. Differentiable optimization layers, deep equilibrium models, and general implicit-differentiation frameworks provide the tools to train such models through converged solves without unrolling the solver \citep{amos2017optnet,bai2019deep,blondel2022efficient}.
In PDE settings, \citet{kochkov2021machine} learn local corrections inside a CFD solver, \citet{you2022learning} learn constitutive operators inside a PDE solve, and \citet{jiao2025one} learn a spatially local operator from a single PDE solution. MEEC-Net uses this local-through-global formulation, with the learned flux assembled as a conservative edge cochain on the MEEC complex.

\noindent\textbf{Meshfree methods in numerical computing.} Many techniques achieve consistency by enforcing a local polynomial reproduction (LPR) property: RKPM \citep{chen2017meshfree}, MLS/GMLS \citep{lancaster1981surfaces, trask2017staggered}, and PSE \citep{schrader2010}\@. \citet{wendland2004scattered} and \citet{mirzaei2012generalized} establish conditions on point clouds under which LPR is well-posed. It is well known \citep{bonet1999variational} that a local construction can achieve either discrete conservation structure or LPR, but not both. Two prior works achieve both with global constructions: \citet{chiu2012conservative} formulates a global convex optimization problem to produce consistent meshfree finite volume stencils, while \citet{trask2020conservative} poses a sequence of weighted graph-Laplacian problems without DEC structure. The present work requires only a single sparse Schur solve with the connectivity of the $\epsilon$-ball graph, and admits a DEC interpretation to access a standard stability analysis for second-order elliptic problems. In scientific machine learning, mesh-based architectures such as MeshGraphNet \citep{pfaff2020learning} learn discretizations end-to-end without exploiting underlying numerical structure. Complementary works build explicitly upon conventional meshfree machinery: GMLS-Nets \citep{trask2019gmlsnets} parameterizes GMLS estimators as learnable layers, while SPNets \citep{schenck2018spnets} and Neural SPH \citep{toshev2024neuralsph} embed SPH kernel summations as differentiable primitives.

\noindent\textbf{Structure-preserving machine learning.} Embedding physical symmetries and conservation laws into learning architectures has consistently improved generalization and data efficiency \citep{bronstein2017geometric, bochev2006principles}. Hamiltonian and Lagrangian networks \citep{greydanus2019hamiltonian} provided early demonstrations that conservation structure can be hard-coded into learned dynamics. Equivariant graph networks \citep{satorras2021n, batzner20223, du2022se} extend this to rotation and translation invariance with strong results in molecular property prediction. Efforts to enforce conservation in neural operators more broadly have often resorted to post-hoc projection or penalty terms \citep{liu2025conservation}, whereas cochain-based methods embed it exactly through topology. In the PDE setting, mimetic and variational discretizations \citep{arnold2018finite} offer the appropriate algebraic framework. \citet{trask2022enforcing} introduced data-driven DEC for conservative PDE learning on meshes; subsequent work developed mortar coupling \citep{jiang2024structure}, conditional neural Whitney forms for digital twins \citep{kinch2025structure}, and structure-preserving geometry-aware surrogates \citep{actor2024data, shaffer2026structure}. These methods demonstrate that conservative cochain structure improves generalization and physical fidelity, but all rely on an underlying mesh. MEEC-Net extends this machinery to point clouds, preserving the structural guarantees while removing mesh dependency.

% ============================================================
\section{Background and Problem Setting}
\label{sec:prelim}

We consider steady-state physical systems governed by a PDE operator $\mathcal{G}$ and boundary operator $\mathcal{B}$ of the form
\begin{equation}
   \label{eq:PDE_operator}
    \mathcal{G}(u,\mu,f) = 0, \; \textrm{in} \; \Omega \subset \mathbb{R}^d,
    \qquad\qquad
    \mathcal{B}(u) = u_b \; \textrm{on} \; \partial \Omega
\end{equation}
where $u:\Omega\to\R^{N_F}$ is the state of a conserved physical field with $N_F$ components, $f:\Omega\to\R^{N_F}$ is an external forcing, $u_b:\partial\Omega\to \R^{N_F}$ is boundary data, and $\mu$ is a scalar, vector, or tensor-valued physical parameter field.
Given a point cloud $\mathcal{X}=\{x_i\}_{i=1}^N\subset\Omega\subset\mathbb{R}^d$ with fill distance $h$, boundary nodes $\mathcal{X}_B = \mathcal{X}\cap\partial\Omega$, and interior nodes $\mathcal{X}_I = \mathcal{X}\setminus\mathcal{X}_B$, let $z=\{\Omega, \mu,f,u_b\}$ denote the problem specification (geometry, spatially varying fields, source terms, and boundary conditions).
We aim to recover a learned version $\mathcal{G}$ from a training set of $N_{\text{train}}$ instances $\mathcal{D}=\{(\mathcal{X}^{(k)},z^{(k)},u_{\textrm{data}}^{(k)})\}_{k=1}^{N_{\text{train}}}$, such that solutions generated by solving the learned model reproduce the training data and generalize to test point clouds that may differ in resolution, domain shape, and physical parameters.

\noindent \textbf{DEC notation.} We summarize here required notation; see Appendix~\ref{app:dec_background} for an introductory DEC primer and \citep{hirani2003discrete, trask2022enforcing} for detailed treatments. A $k$-cochain $\sigma\in C^k$ represents an integrated physical quantity over oriented $k$-cells (points, edges, faces). Coboundaries $d_k:C^k\to C^{k+1}$ encode topology via incidence matrices and satisfy $d_{k+1}\circ d_k=0$ (discrete Stokes theorem); Hodge stars $M_k$ are diagonal, positive matrices encoding metric information (volumes, lengths, areas). We discretize the conservation law $-\nabla\cdot\boldsymbol{\sigma}=f$ as $d_0^\top M_1\sigma = M_0 f$, for 1-cochain $\sigma$ and 0-cochain $f$. Boundary handling (Appendix~\ref{app:bcs}) excludes boundary nodes from the residual equation, so for any $\sigma$ and $M_1$ the interior-summed residual equals the net boundary flux and vanishes under zero-flux boundary data, expressing discrete global conservation. For the current work we focus on $k=0,1$, for which the relevant codifferential is $\delta_1 = M_0^{-1}d_0^\top M_1$; in the classical setting $M_0$ and $M_1$ would be diagonal matrices of dual cell volumes and primal/dual edge-measure ratios, respectively.

% ============================================================
\section{Method}
\label{sec:method}

The following pipeline summarizes the full construction bridging point clouds to learned models. Sections~\ref{sec:mfdec} and~\ref{sec:learned} define subproblems. The entire pipeline is end-to-end differentiable, allowing gradients to propagate from the solution error through to the geometric representation. 
\begin{equation}
    \label{eq:pipeline}
    \underbrace{\mathcal{X}}_{{\text{point cloud}}}
    \;\xrightarrow{\;\text{\S\ref{sec:mfdec}}\;}\;
    \underbrace{(d_0,\,M_0,\,M_1)}_{\text{MEEC complex}}
    \;\xrightarrow{\;\text{\S\ref{sec:learned}}\;}\;
    \underbrace{(\xi_e,\,\mathcal{K}_\theta)}_{\text{local flux model}}
    \;\xrightarrow{\;\text{solve}\;}\;u
\end{equation}

% ============================================================
\subsection{Meshfree exterior calculus (MEEC)}
\label{sec:mfdec}
\begin{wrapfigure}{r}{0.38\textwidth}
    \vspace{-\intextsep}
    \centering
    \includegraphics[width=\linewidth]{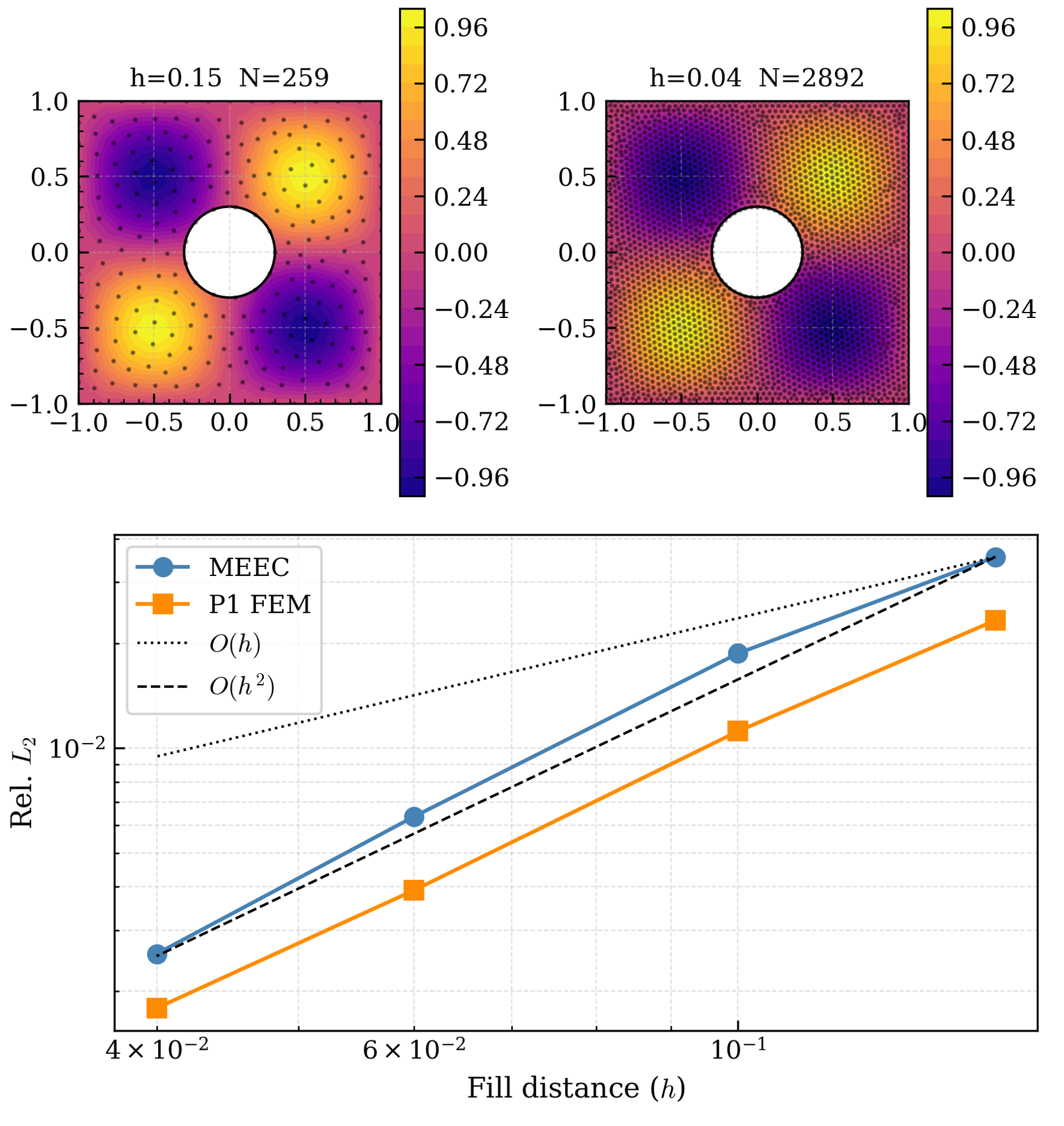}
    \caption{Convergence under mesh refinement for the Poisson equation. Both our meshless DEC and conventional FEM exhibit $O(h^2)$ solution convergence, with comparable error constants.}
    \label{fig:poisson_conv}
\end{wrapfigure}
Given $\mathcal{X}$ sampling $\Omega$, we construct edges defining an $\epsilon$-ball graph, $E=\{(i,j):i<j,\;\|x_i-x_j\|<\epsilon\}$.
For oriented edge $e=(i,j)$ we define displacement $\delta x_e = x_j - x_i \in \mathbb{R}^d$, length $r_e = \|\delta x_e\|$, midpoint $\bar x_e = \tfrac12(x_i+x_j)$, unit tangent $\hat e_e = \delta x_e / r_e$, and orientation $s_e(k)=+1$ if $k=j$, $-1$ if $k=i$, and $0$ otherwise.
We associate virtual volumes $m_i$ and areas $a_e$ with nodes and edges, posing the following ansatz that mimics conventional DEC and defines a discrete de~Rham complex on the point cloud:
\begin{equation}
    (\delta_1 \sigma)_i = \frac{1}{m_i}\sum_{e \sim i} s_e(i)\ a_e \int_{e} \sigma \cdot d\ell,
    \quad
    C^0(V)\;\mathrel{\overset{d_0}{\underset{\delta_1}{\rightleftarrows}}}\;C^1(E).
    \label{eq:complex}
\end{equation}
Equivalently, with $M_0 = \operatorname{diag}(m_i)$, $M_1 = \operatorname{diag}(a_e)$, and $(d_0 u)_e := u_j - u_i$, this reads $\delta_1 \sigma = M_0^{-1}d_0^\top M_1 \sigma$.
Given a compactly supported RBF kernel $\phi(r) = (1-r/\epsilon)^2_+$, set $\kappa_i = \sum_{e \sim i}\phi(r_e)$ and define virtual volumes by $m_i = \frac{1/\kappa_i}{\sum_j 1/\kappa_j}|\Omega|$ for $i\in\mathcal{X}_I$, noting that $\sum_i m_i = |\Omega|$. We do not associate volume with nodes on the boundary. To define virtual areas, we apply Mirzaei's LPR theory \citep{mirzaei2012generalized} to the divergence operator. A discrete stencil $D^\alpha_h u_i = \sum_j \alpha_{ij}\, u_j$ provides a \emph{local polynomial reproduction of degree $m$} on $\Omega\subset\mathbb{R}^d$ if there exist constants $h_0, C_1, C_2 > 0$ such that for every point cloud $X\subset\Omega$ with fill distance $h\le h_0$:
\begin{enumerate}[label=(\roman*),leftmargin=*,topsep=2pt,itemsep=1pt]
    \item \emph{Polynomial reproduction}: $\sum_j \alpha_{ij}\, p(x_j) = (D^\alpha p)(x_i)$ for all $p\in\mathbb{P}_m^d$.
    \item \emph{Boundedness}: $\sum_j |\alpha_{ij}| \le C_1\, h^{-|\alpha|}$.
    \item \emph{Locality}: $\alpha_{ij} = 0$ whenever $\|x_j-x_i\| > C_2\, h$.
\end{enumerate}
Mirzaei's main result \citep[Theorem 4.3]{mirzaei2012generalized} bounds the truncation error of any such stencil: if $\{\alpha_{ij}\}$ satisfies (i)--(iii) for some $m$, then $|(D^\alpha u)(x_i) - D^\alpha_h u_i| \le C\, h^{m+1-|\alpha|}\, |u|_{C^{m+1}}$ for $u\in C^{m+1}$. We restate the classical theorem formally as Theorem~\ref{thm:lpr} in Appendix~\ref{app:lpr_proof}, where we also specialize it to $\delta_1$.

For $\delta_1$, the LPR conditions reduce at each interior node to linear equalities in $\{a_e\}$ (Appendix~\ref{app:qp}). We enforce them via the equality-constrained quadratic program
\begin{equation}
    \min_{a \in \mathbb{R}^{|E|}} \frac{1}{2} \sum_e a_e^2 / \phi_e \qquad \text{s.t.} \qquad \delta_1 p(x_i) = -\nabla\cdot p(x_i) \; \forall p\in{(\mathbb{P}_1^d)}^d,\; \forall i.
    \label{eq:opt}
\end{equation}
Condition (i) is enforced strongly, (iii) holds via kernel support, and (ii) is minimized; the unique minimizer is $a = \Phi B^\top \lambda$ with $\lambda$ solving the sparse Schur complement system $S\lambda = c$, $S := B\Phi B^\top$, $\Phi := \operatorname{diag}(\phi_e)$. Theorem~\ref{thm:lpr} then yields an $O(h)$ truncation rate on $\delta_1$, with a direct proof in Appendix~\ref{app:lpr_proof}. Optionally, enforcing $a_e \ge 0$ in the optimization guarantees $M_1$ induces a valid Hodge star at the cost of differentiability; though we observe some negative $a_e$, we observe no performance difference and leave it off by default. 

\noindent\textbf{Properties.} The construction provides exact discrete conservation in the sense of Section~\ref{sec:prelim} (interior-summed residual equals net boundary flux for any $\sigma$, by the algebraic identity $d_0\mathbf{1}=0$), $O(h)$-consistent divergence (Theorem~\ref{thm:lpr}), and end-to-end differentiability (Corollary~\ref{cor:area-regularity}).

\noindent\textbf{Hodge Laplacian.} Following standard DEC \citep[Thm.~3.5]{trask2022enforcing}, the Hodge Laplacian $\Delta_0 := \delta_1 d_0 = M_0^{-1}d_0^\top M_1 d_0$ is symmetric positive semidefinite in the $M_0$-inner product with a constant vector null-space; consistency reduces to that of $\delta_1$ since $d_0 u$ is exact in the cochain sense. Figure~\ref{fig:poisson_conv} reports $O(h^2)$ \emph{solution-level} convergence on a Dirichlet Poisson problem with error constants comparable to FEM; boundary handling is detailed in Appendix~\ref{app:bcs}.

% ============================================================
\subsection{Learned conservation law}
\label{sec:learned}
We seek a conservation law $-\nabla\cdot\boldsymbol{\sigma}(u;\mu)=f$ where the constitutive relation $\boldsymbol{\sigma}$ is unknown. Boundary conditions are enforced by prescribing $u$ or $\boldsymbol{\sigma}$ on $\mathcal{X}_B$, and we pose the ansatz
\begin{equation}
  \label{eq:system}
  \epsilon\,d_0^\top M_1 d_0\,u \;+\; d_0^\top M_1\mathcal{F}_\theta(u;\mu) = M_0 f,
\end{equation}
where $\mathcal{F}_\theta(u;\mu)\in\mathbb R^{|E|\times N_F}$ is the learned flux 1-cochain and $\epsilon>0$ is a fixed background diffusion that ensures well-posedness. 
Conservation holds for any $\theta$ by the topological identity of Section~\ref{sec:prelim}. Learning is restricted to identifying the constitutive law through $\mathcal{F}_\theta$, stabilized by $\Delta_0$.

We define the learned flux through a local first-order constitutive law depending on $u$ and $\nabla u$ at each edge. For each \(e=(i,j)\), we construct invariant edge-frame features \(\xi_e\) by projecting scalar, vector, and tensor quantities from the physical field and conditioning variables into the local oriented edge frame. We decompose the per-node state and parameters into channels $u_i = \{q_i^m\}_{m=1}^{N_u}$ and $\mu_i = \{p_i^m\}_{m=1}^{N_\mu}$, with each $q_i^m, p_i^m \in \R,\; \R^d,\text{ or }\R^{d\times d}$.
We construct the inputs by stacking the edge-projected features over these fields as \(\xi_e = \operatorname{stack}_{a\in u,\mu} \Pi_e\left( a \right)\) where $\Pi$ is the appropriate projection operation (see Appendix~\ref{app:pullback}). For a scalar field \(p\) this is,
\[
  \xi_e^p := \left( \frac{p_i+p_j}{2}, \frac{p_j-p_i}{r_e}\right)
  =
  \left(p(\bar x_e),\nabla p(\bar x_e)\cdot \hat e_e\right)+O(r_e^2),
\]
where \(\bar x_e\), \(r_e\), and \(\hat e_e\) are the edge midpoint, length, and unit tangent. Vector and tensor fields are handled analogously by contracting their components and edge-directional derivatives with the local edge frame; the full construction is given in Appendix~\ref{app:pullback}. Thus \(\xi_e\) is a consistent edge-frame approximation of the local first-order PDE variables \((u,\nabla u)\).

The learned constitutive law \(\mathcal K_\theta\), instantiated as a feedforward MLP with Tanh activations (architecture details in Appendix~\ref{app:arch}), maps these edge features to a scalar, or componentwise vector-valued, flux density along the oriented edge; we write $\hat\sigma_\theta$ for the continuum one-form induced by $\mathcal K_\theta$ via the edge-frame map. The corresponding discrete edge flux is defined by midpoint quadrature:
\[
  \mathcal F_\theta(u;\mu)_e
  =
  r_e\,\mathcal K_\theta(\xi_e)
  =
  \int_e \hat\sigma_\theta\ + O(r_e^3). 
\]
The factor \(r_e\) is the midpoint quadrature approximation of the edge integral, which defines \(\mathcal F_\theta(u;\mu)_e\) as a valid discrete \(1\)-cochain.

\begin{proposition}[Edge-frame invariance]
\label{prop:invariance}
By construction, $\xi_e$ depends only on $r_e$, $\hat{e}_e$, and frame projections onto $(\hat{e}_e,\hat{n}_e)$, hence is invariant under translation and global rotation.
\end{proposition}
\begin{proposition}[Consistency of the learned edge flux]
\label{prop:flux_consistency}
Let $\hat\sigma=\hat\sigma(u,\nabla u)$ be a smooth first-order continuum one-form in the class represented by the edge-frame variables $\xi(x,\hat e)$, and let $\mathcal{K}^\star$ denote its scalar edge density, so that $\hat\sigma(x)(\hat e)=\mathcal{K}^\star(\xi(x,\hat e))$. By consistency of the edge encoder, $\xi_e=\xi(\bar x_e,\hat e_e)+O(r_e^2)$. If $\mathcal K_\theta$ is Lipschitz with uniform density error $\gamma_\epsilon := \sup_\xi |\mathcal K_\theta(\xi)-\mathcal{K}^\star(\xi)|$, then the learned edge cochain $\mathcal F_\theta(u;\mu)_e = r_e\,\mathcal K_\theta(\xi_e)$ satisfies
\[
    \mathcal F_\theta(u;\mu)_e = \int_e \hat\sigma + O(r_e^3) + r_e\,\gamma_\epsilon.
\]
Hence $\mathcal F_\theta(u;\mu)$ is consistent with the target continuum one-form, with kernel error $\gamma_\epsilon$.
\end{proposition}

\begin{proof}
By encoder consistency and Lipschitz continuity of the learned model,
\[
    \mathcal K_\theta(\xi_e) = \hat\sigma(\bar x_e)(\hat e_e) + O(r_e^2) + O(\gamma_\epsilon).
\]
Multiplying by \(r_e\) and using midpoint quadrature,
\(\int_e\hat\sigma=r_e\hat\sigma(\bar x_e)(\hat e_e)+O(r_e^3)\),
gives the result.
\end{proof}

\begin{wrapfigure}{r}{0.38\textwidth}
    \vspace{-\intextsep}
    \centering
    \includegraphics[width=\linewidth]{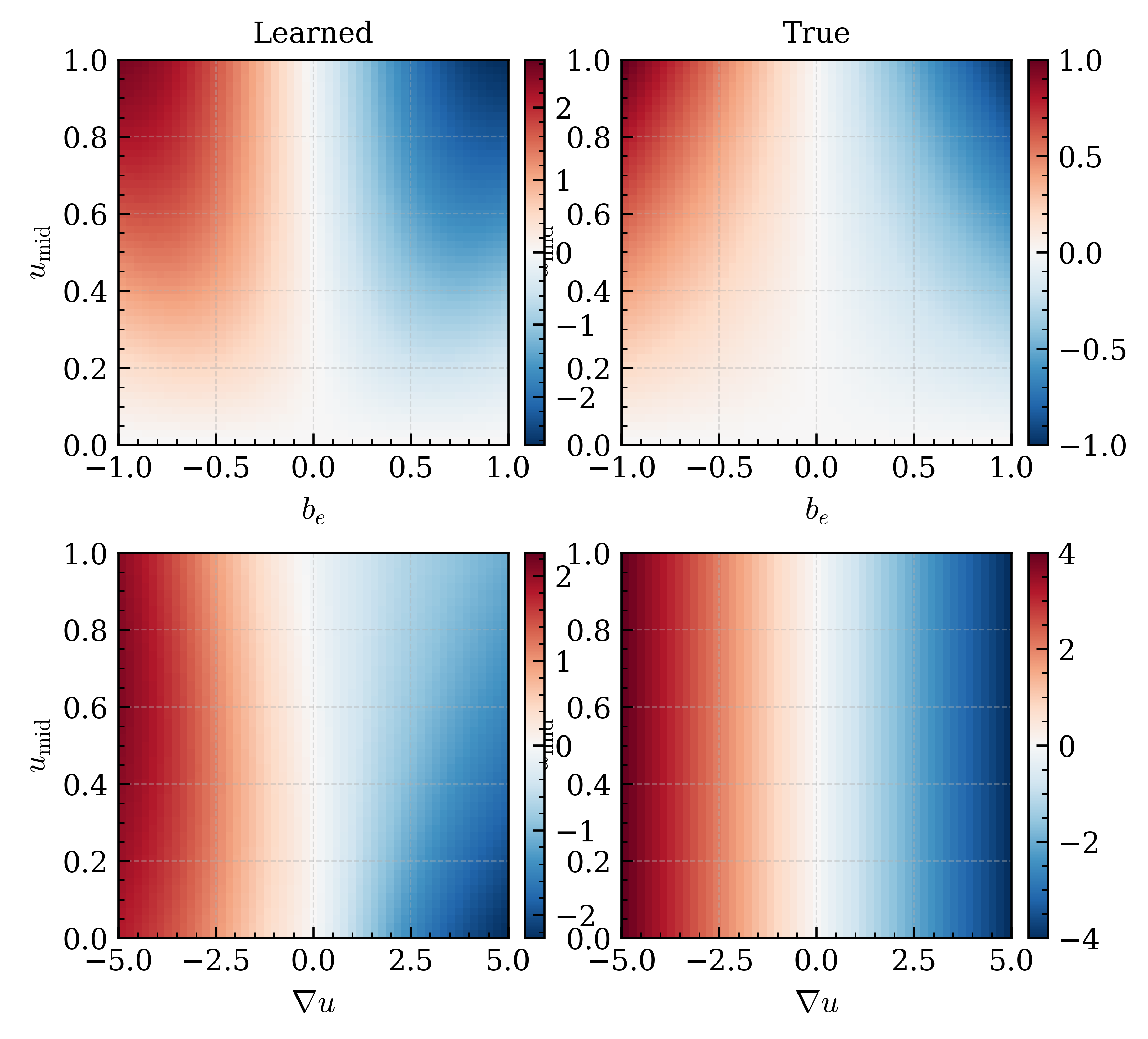}
    \caption{Learned flux field for advection-diffusion recovered from a single training solution, where $b$ is the edge aligned velocity input. The plots are shown along slices corresponding to $b=0$ and $\nabla u=0$. The model qualitatively identifies the advective and diffusive contributions across the edge-feature space, enabling generalization to unseen geometries, boundary conditions, and advection velocities without retraining. See Remark~\ref{rem:identifiability} for comments on flux identifiability.}
    \label{fig:flux_recovery1}
\end{wrapfigure}
Proposition~\ref{prop:flux_consistency} is the mechanism behind generalization within a PDE family. Provided $\mathcal{K}_\theta$ approximates the target continuum density $\mathcal{K}^\star$ on a compact feature set $\Xi$, the same learned law induces a consistent edge cochain on any $\mathcal{X}$ whose features lie in $\Xi$. The remaining requirement is that the training data sufficiently identify $\mathcal{K}^\star$ over that range (see Remark~\ref{rem:identifiability}). Our experiments show that this can be achieved from as few as a single training instance, enabling the data-efficient generalization demonstrated in Section~\ref{sec:results}.

% ============================================================
\subsection{Analysis of learned model}
\label{sec:theory}

Sections~\ref{sec:mfdec} and~\ref{sec:learned} together define system~\eqref{eq:system}, an end-to-end trainable conservation law on any point cloud. The central result decomposes the solution error into a discretization residual $C_{\mathrm{disc}}\,h$  following Theorem~\ref{thm:lpr} and a bounded kernel error $\gamma_\epsilon$, with constants depending upon a fixed regularity class of point clouds and feature distributions in $\Xi$. 

Let $K:=\epsilon d_0^\top M_1 d_0$, $b:=M_0 f$, and $N_\theta(u;\mu):=d_0^\top M_1\mathcal F_\theta(u;\mu)$, so the learned residual is $G_\theta(u;\mu):=Ku+N_\theta(u;\mu)-b$ with $G_\theta(u_\theta;\mu)=0$. 

\begin{proposition}[Well-posedness and stability]
\label{prop:stability}
Suppose $K$ is positive definite on the interior $\mathcal{X}_I$ and $N_\theta$ is Lipschitz in $u$ at fixed $\mu$ with constant $C_\theta$. If $\tau := \|K^{-1}\|C_\theta < 1$, then $G_\theta(\cdot;\mu)=0$ has a unique solution $u_\theta$ and
\[
    \|u_\theta - v\| \le C_{\mathrm{stab}}(h)\, \|G_\theta(v;\mu)\|, \qquad C_{\mathrm{stab}}(h) := \tfrac{\|K^{-1}\|}{1-\tau},
\]
for any $v$ matching the Dirichlet data. By analogy with the continuous Lax-Milgram bound $\|K^{-1}\| \le 1/(\epsilon\,\alpha_P)$ via the discrete Poincar\'e constant $\alpha_P$, the contraction $\tau<1$ holds when $C_\theta < \epsilon\,\alpha_P$, i.e., the learned-flux Lipschitz constant is small relative to the diffusive coercivity; discrete operator norms in the MEEC setting may carry residual $h$-dependence, hence the notation $C_{\mathrm{stab}}(h)$ (see Appendix~\ref{app:stability_consistency}).
\end{proposition}

\begin{theorem}[Generalization error decomposition]
\label{thm:main}
Let $u^\star$ solve the continuum PDE and $I_h u^\star$ its nodal interpolant; let $G_h^\star$ be the residual~\eqref{eq:system} evaluated with the exact-kernel discrete flux $F_h^\star$. Suppose $\|G_h^\star(I_h u^\star;\mu)\| \le C_{\mathrm{disc}}\,h$, and let $\gamma_\epsilon := \sup_{\xi\in\Xi}|\mathcal K_\theta(\xi) - \mathcal K^\star(\xi)|$. Then, with $C_{\mathrm{flux}}$ the discrete-divergence operator norm (Lemma~\ref{lem:flux_to_residual}, Appendix~\ref{app:stability_consistency}),
\[
    \|u_\theta - I_h u^\star\| \le C_{\mathrm{stab}}(h)\!\left( C_{\mathrm{disc}}\,h + C_{\mathrm{flux}}\,\gamma_\epsilon \right).
\]
The discretization-consistency hypothesis follows from Theorem~\ref{thm:lpr} and Proposition~\ref{prop:flux_consistency} applied to the continuum equation. Proof in Appendix~\ref{app:stability_consistency}.
\end{theorem}
\begin{remark}[Universal approximation]
\label{rem:universal}
Theorem~\ref{thm:main} decomposes solution error into two contributions: a discretization error $C_{\mathrm{stab}}(h)\,C_{\mathrm{disc}}\,h$ and a nonlinear flux model form error $C_{\mathrm{stab}}(h)\,C_{\mathrm{flux}}\,\gamma_\epsilon$. This establishes that the solution error can be made arbitrarily small if (1) a discrete Lax equivalence theorem~\citep{LaxRichtmyer1956} holds, so that the first term goes to $0$ as $h\to 0$, and (2) MEEC-Net has sufficient capacity and training for the second term to vanish.
\end{remark}

\begin{remark}[Generalization and identifiability from a single solution]
\label{rem:identifiability}
Since $\mathcal{K}_\theta$ depends only on edge-frame features $\xi_e$, not on $z=\{\Omega,\mu,f,u_b\}$, Theorem~\ref{thm:main} transfers the same law to new geometries, boundary data, forcings, and parameters whenever $C_{\mathrm{stab}}(h), C_{\mathrm{disc}}, C_{\mathrm{flux}}$ are uniformly bounded and features lie in $\Xi$. Although a single solution underdetermines the flux up to divergence-free perturbations, the shared MLP class empirically recovers $\mathcal{K}^\star$ from one training instance when training features cover the test distribution (Figure~\ref{fig:flux_recovery1}, \S\ref{sec:results}).
\end{remark}

% ============================================================
\subsection{Implementation}
\label{sec:implementation}

\textbf{Training.}
Given a dataset of one or more problem instances and corresponding target solutions, we minimize $\mathcal{L}(\theta)=\sum_k \| u_\theta^{(k)}-u^{(k)} \|_{2}^2$, where $u_\theta^{(k)}$ solves~\eqref{eq:system} for parameters $\theta$ and instance $k$. Gradients $\partial\mathcal{L}/\partial\theta$ are computed via the implicit function theorem from the converged $u_\theta$.

\textbf{Newton solve and scaling.}
Because~\eqref{eq:system} is nonlinear, each forward pass is a Newton solve; Proposition~\ref{prop:stability} ensures existence, uniqueness, and well-posedness for any $\theta$, provided the necessary stability conditions are met. The Schur system $S\lambda = c$ from~\eqref{eq:opt} inherits the sparsity of the $\epsilon$-ball graph, with $P\times P$ diagonal blocks $S_{ii}$ ($P=d(d+3)/2$, i.e.\ $5$ in 2D, $9$ in 3D) and off-diagonal blocks $S_{ij}$ supported only on graph edges. Jacobian assembly is $O(N)$ on the $\epsilon$-ball graph; the resulting sparse linear system is solved with a sparse direct factorization.

% ============================================================
\section{Results}
\label{sec:results}

We evaluate MEEC-Net in sparse data regimes where local, invariant flux learning is hypothesized to enable single-shot learning. These regimes are common in engineering design, where high-fidelity training data are costly. 
In all experiments we train and evaluate with normalized $L_2$ error. Full hyperparameter specifications and training setups are provided in Appendix~\ref{app:exp_setup}. All baselines are trained on the same train/test splits, point clouds, boundary conditions, forcing terms, and physical parameters as MEEC-Net. We input boundary conditions and forcing terms as additional input channels, along with spatial coordinates and boundary identification (see Appendix~\ref{app:baselines}). Where applicable, baselines are given the same graph connectivity and the same per-node input features, including coordinates, physical parameters, forcing, and boundary-condition indicators. We use comparable training budgets across methods and report each model using the checkpoint with the lowest training error.

\begin{figure}[!t]
    \centering
    \includegraphics[width=1.0\linewidth]{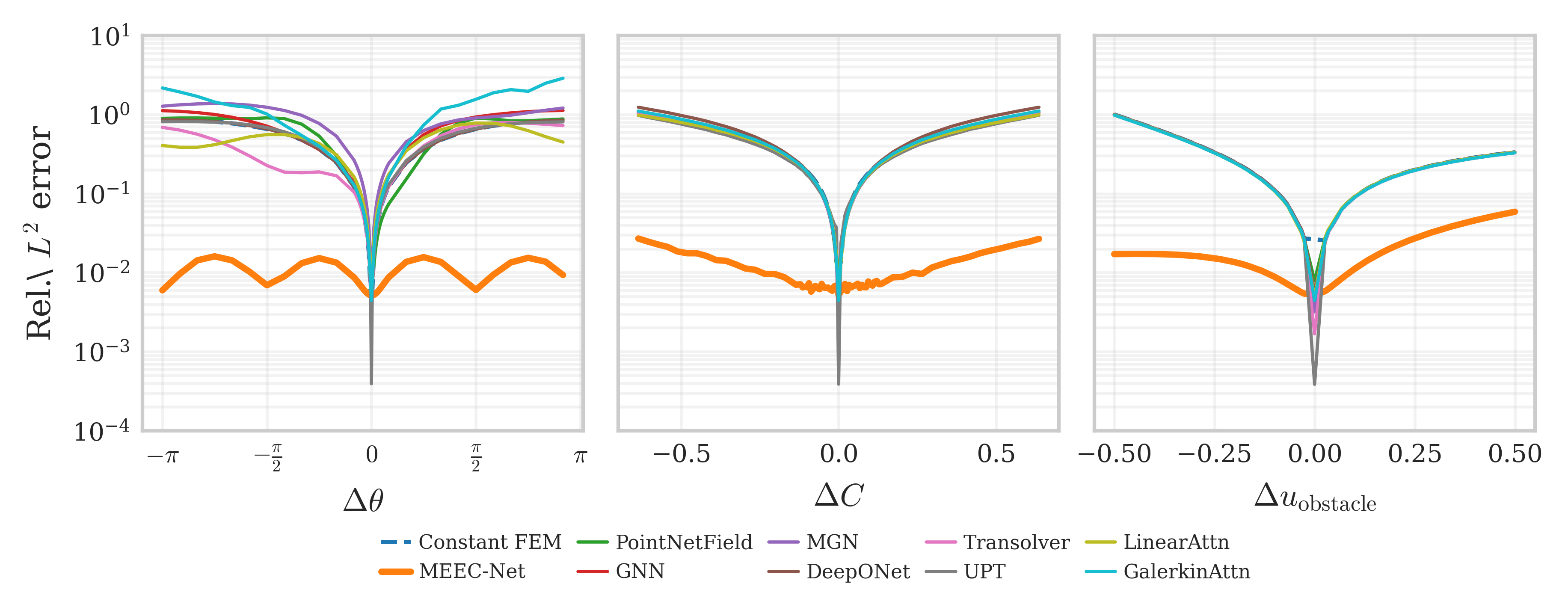}
    \caption{\textbf{Single-shot generalization for advection-diffusion.} Trained from one data point, the learned model generalizes across variations in geometry, physical parameters (velocity), and boundary conditions. While direct prediction baselines may fit the single training sample, they are unable to extrapolate meaningfully. From left to right, $\Delta \theta$ is the variation from the training velocity direction which results in variation in $\mu$, $\Delta C$ is a translation of the coordinate of the obstacle center point along the diagonal as a geometric perturbation, and $\Delta u_{\rm{obstacle}}$ is a change in the boundary value enforced on the obstacle (from $u_{\rm{obstacle}}=1$). } 
    \label{fig:sweeps_v2}
\end{figure}
\noindent\textbf{Single-shot physics recovery.} We demonstrate single-shot physics recovery on canonical PDEs. Specifically, we train on a single solution to the advection-diffusion equation with known BCs and forcing, but unknown governing equation. Figure~\ref{fig:flux_recovery1} shows qualitative flux recovery for this problem, where the analytic flux is reproduced by MEEC-Net. Figure~\ref{fig:sweeps_v2} shows the robustness of this learned model to variations in physical parameters, geometry, and boundary conditions; conventional neural operators fail to generalize in the single-shot setting.

\noindent\textbf{Data-efficient training.} Figure~\ref{fig:data_efficiency} reports OOD relative-$L_2$ error versus $N_{\text{train}}$ (training-set size) for advection-diffusion and linear elasticity. Across both problems MEEC-Net achieves $1$--$2$ orders of magnitude lower error than baselines; baselines need $>\!10$ samples on advection-diffusion before predictions become even mildly informative. MEEC-Net error is nearly constant in $N_{\text{train}}$ on advection-diffusion, evidencing single-shot recovery, while on linear elasticity (where forcing varies more across instances) MEEC-Net error decreases modestly from $N_{\text{train}}=1$ to $N_{\text{train}}=10$.

\begin{figure}[!t]
    \centering
    \includegraphics[width=0.9\linewidth]{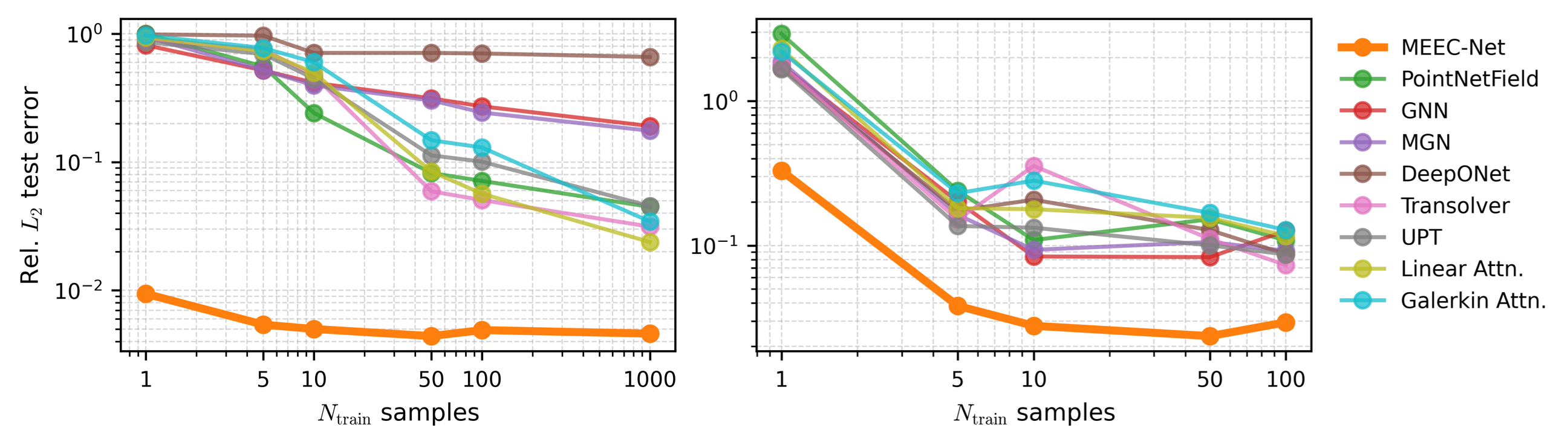}
    \caption{\textbf{Data-efficient training.} Our approach achieves substantially improved data efficiency, shown here for advection-diffusion \textit{(left)} and elasticity \textit{(right)}. MEEC-Net achieves $1$--$2$ orders of magnitude lower error than baselines across all training set sizes, and demonstrates single-shot recovery on advection-diffusion.}
    \label{fig:data_efficiency}
\end{figure}

\noindent\textbf{Canonical steady-state PDEs.} We benchmark MEEC-Net on advection-diffusion, Darcy flow, linear elasticity, and nonlinear variants of the first two (full setup in Appendix~\ref{app:exp_setup}). For each PDE, training and in-distribution testing use a unit square with a single hole; OOD testing uses geometries with multiple variable-sized holes. Each PDE also includes per-instance physical variability with an expanded range for OOD\@: advection direction $\theta$ for AD, log-normal diffusivity amplitude $\kappa$ for Darcy, and prescribed boundary displacement direction for elasticity. Table~\ref{tab:canonical_results} reports ID and OOD relative-$L_2$ errors.

\begin{table}[h]
  \centering
  \footnotesize
  \caption{Normalized $L_2$ errors (\%) on canonical PDE benchmarks in low-data setting with $N_{\text{train}}=100$. MEEC substantially outperforms all baselines in both in-distribution (ID) and out-of-distribution (OOD) evaluations. The varied OOD factors are listed below each experiment.}\label{tab:canonical_results}
  \begin{tabular}{lllllllllll}
    \toprule
    ($1e-2$) & \multicolumn{2}{c}{A.D.} & \multicolumn{2}{c}{Darcy} & \multicolumn{2}{c}{Elas.} & \multicolumn{2}{c}{Nonl. A.D.} &  \multicolumn{2}{c}{Nonl. Darcy} \\
    Method     & ID & OOD & ID & OOD & ID & OOD & ID & OOD & ID & OOD  \\
     &  & $\Omega_g,\mu$ &   & $\Omega_g,\mu$ &   & $\Omega_g,u_b$ &   & $\Omega_g,\mu$ &   & $\Omega_g,\mu$ \\
    \midrule
    PointNet &  \cellcolor{red!10} 7.13 & 45.84 & 5.13 & \cellcolor{red!10} 34.86 & 10.85 & 32.39 & 10.67 & 45.95 & 6.20 & 35.22 \\
    GNN & 27.21 & \cellcolor{red!10} 40.32 & 12.45 & 34.97 & 12.67 & 25.12 & 29.30 & \cellcolor{red!10} 40.96 & 13.97 & \cellcolor{red!10} 26.51 \\ 
    MGN & 24.36 & 41.63 & 11.75 & 31.27 & 9.19 & 14.51 & 31.35 & 46.29 & 12.89 & 36.08 \\
    Transolver & 5.06 & 52.42 & 5.86 & 48.71 & \cellcolor{red!10} 7.47 & 13.65 & \cellcolor{red!10} 9.27 & 50.97 & 6.13 & 48.39 \\
    UPT & 10.08 & 52.36 & 7.97 & 42.50 & 8.68 & 12.90 & 15.38 & 51.89 & 6.58 & 45.44 \\
    Linear Attn. & 5.67 & 52.09 & \cellcolor{red!10} 5.22 & 33.79 & 11.64 & 16.50 & 12.76 & 52.36 & \cellcolor{red!10} 5.50 & 30.32 \\
    Galerkin Attn. & 13.01 & 56.51 & 8.75 & 46.11 & 10.00 & \cellcolor{red!10} 12.84 & 30.74 & 58.02 & 9.26 & 36.37 \\
    \midrule
     \textbf{MEEC-Net} & \cellcolor{blue!5} \textbf{0.49} &  \cellcolor{blue!5} \textbf{1.81} &  \cellcolor{blue!5} \textbf{0.56} &  \cellcolor{blue!5} \textbf{0.57} &  \cellcolor{blue!5} \textbf{2.94} &  \cellcolor{blue!5} \textbf{2.68} &  \cellcolor{blue!5} \textbf{0.64} &  \cellcolor{blue!5} \textbf{1.57} &  \cellcolor{blue!5} \textbf{0.53} &  \cellcolor{blue!5} \textbf{0.71}\\
    \bottomrule
  \end{tabular}
\end{table}

\begin{wrapfigure}{r}{0.38\textwidth}
    \vspace{-\intextsep}
    \centering
    \includegraphics[width=\linewidth]{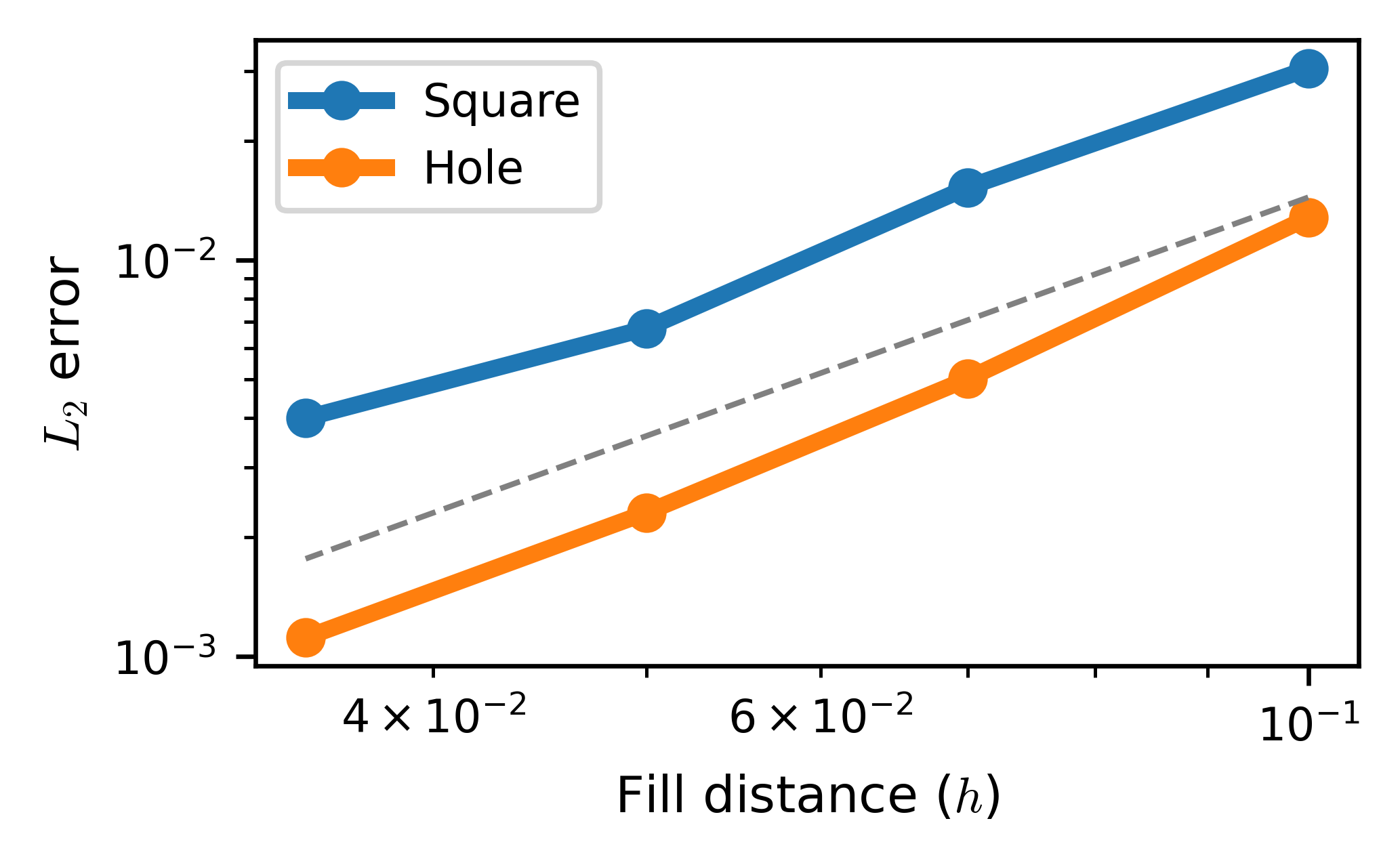}
    \caption{\textbf{Convergence under geometry extrapolation for well trained flux.} Solution error on two unseen geometries for a pre-trained flux map, demonstrating $O(h^2)$ convergence (dashed line) under refinement.}
    \label{fig:learned_flux_convergence}
\end{wrapfigure}
\noindent\textbf{Consistent discretization.} To evaluate discretization error under controlled model error, we consider a problem with a known governing transport law. The flux kernel is pretrained on synthetic feature samples $\xi$, without any PDE solve, to approximate the true constitutive map $\mathcal{K}^\star$ to a prescribed accuracy $\gamma_\epsilon$. We then evaluate the pretrained model across a sequence of discretizations with decreasing $h$, comparing against the manufactured solution (Figure~\ref{fig:learned_flux_convergence}). Without retraining, the method recovers $O(h^2)$ convergence, with accuracy determined solely by $\gamma_\epsilon$ rather than by $h$. This confirms the theoretical separation between discretization and model error, and holds across both domain geometries considered.

\begin{wrapfigure}{r}{0.5\textwidth}
    \centering
    \includegraphics[width=\linewidth]{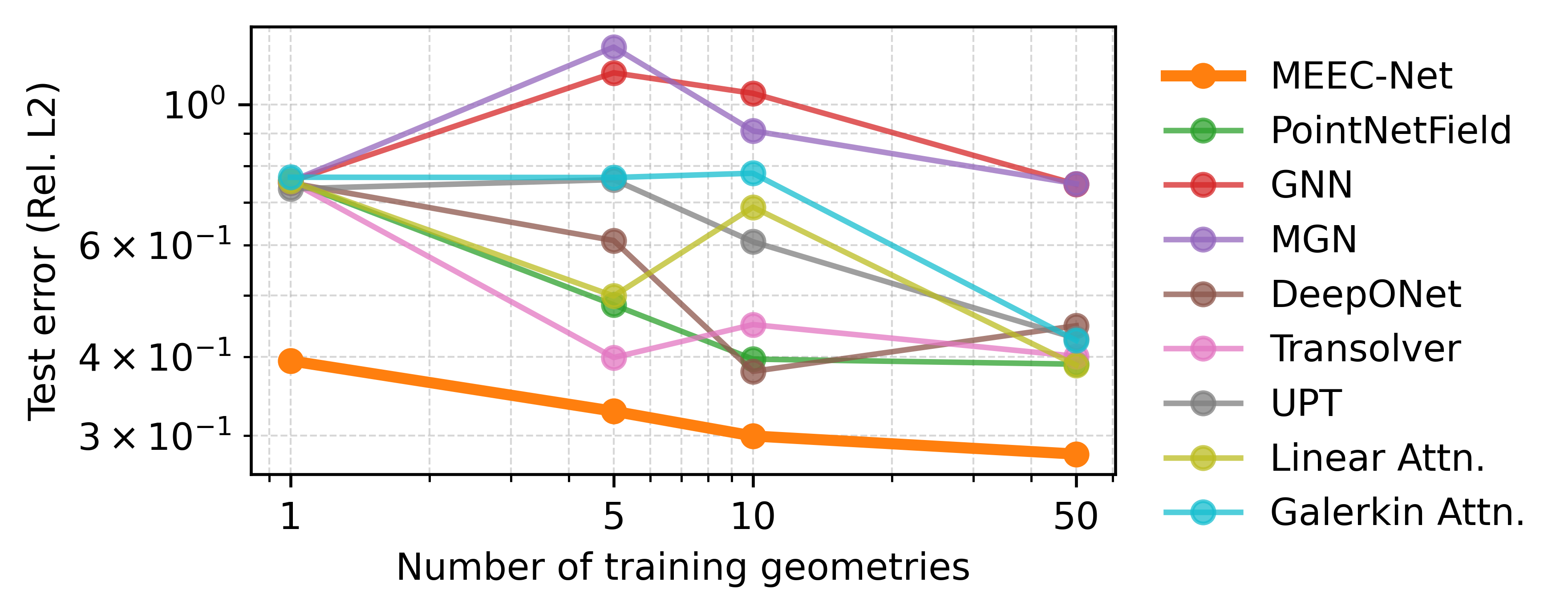}
    \caption{\textbf{Application to complex 3D geometries.} Our approach provides improved performance for the engineering-relevant SimJEB displacement prediction task with much less data than is typically required.}
    \label{fig:JEB_data_efficiency}
\end{wrapfigure}

\noindent\textbf{Model for jet engine bracket.} We evaluate on the SimJEB dataset~\citep{whalen2021simjeb} of $381$ structural bracket geometries from an open engineering design competition, an unusually large corpus for engineering design tasks. We train to predict nodal displacements under a prescribed load and report relative $L_2$ error across held-out geometries. MEEC-Net outperforms all baselines in the low-data regime (Figure~\ref{fig:JEB_data_efficiency}).
% Table~\ref{}, do we want a table?
Methods reporting competitive results on this class of problems typically rely on the DeepJEB augmented dataset~\citep{hong2025deepjeb}, which supplements the original geometries with large volumes of synthetically generated variants, and still require ${\sim}2{,}000$ training samples to reach errors of 20--30\% \citep{rowbottom2026multi, liu2026geometry}. The dependence on such augmentation highlights a fundamental limitation for practical deployment: in real engineering workflows, even a few hundred high-fidelity simulations represent a substantial cost, and further synthetic augmentation may not be feasible. Our method achieves competitive accuracy for this task on a fraction of the original dataset, without augmentation, and accordingly is better aligned with the data-availability constraints of real design problems.

% ============================================================
\section{Conclusion}
\label{sec:conclusion}
We present MEEC and a consistent learnable flux model (MEEC-Net) to learn generalizable physics natively on unstructured point cloud data. We demonstrate single-shot recovery and generalization over geometry, boundary conditions, and forcings on canonical and engineering-relevant design tasks. We develop theoretical error estimates across a given PDE family, based only on the flux recovery and MEEC discretization error, independent of problem specification, which underpins the broad generalization observed empirically. This work shows that the exterior calculus provides scaffolding to separate the local, generalizable constitutive rules from the global, geometry-specific assembly.
% This kind of
This generalization from simple examples to complex emergent behavior is essential for AI tools that reason about and design within the interplay of geometry, physics, and data.

\paragraph{Limitations.} The current construction targets steady-state elliptic problems.
\label{sec:limitations}
The same spatial discretization could be combined with time integration for unsteady problems, but are omitted for simplicity.
As shown in Appendix~\ref{app:results}, the cost of evaluating learned MEEC-Net models is comparable to a conventional solver, and thus provides primary value in data-driven extraction of models; it should not be interpreted as a surrogate providing orders of magnitude speedup, although coarsening techniques like \cite{kinch2025structure} could be trivially applied for acceleration. As discussed in Remark~\ref{rem:shortcomplex}, we work only with the $0$-/$1$-form portion of the de Rham complex, as higher-order forms are prohibitively expensive. The learned constitutive law is restricted to the local first-order class $\mathcal{K}_\theta(u,\nabla u)$, and genuinely nonlocal or higher-derivative closures are not covered. The $O(h^2)$ solution-level convergence reported in Section~\ref{sec:results} is empirical and attributed to elliptic supraconvergence \citep{kreiss1986supraconvergent}; all analysis establishing $O(h)$-convergence is for truncation error only. All bounds presume quasi-uniform point clouds with bounded edge degree; we did not analyze severely anisotropic distributions in this work. 

\paragraph{Reproducibility.} All experiments use a consistent experimental setup and hyperparameter choices, as reported in Appendix~\ref{app:arch} and Appendix~\ref{app:exp_setup}, including the baseline methods.
% Code and data are available at \TODO{insert public repo URL}.

\begin{ack}
The authors acknowledge support from the Department of Energy, Office of Science, Advanced Scientific Computing Research (ASCR) program SEA-CROGS MMICCs center (DE-SC0023191), as well as Army Research Office support through the "Complexity, Nonlocality, and Uncertainty in Heterogeneous Solids" MURI program (W911NF-24-2-0184). B.S. is supported by the National
Science Foundation Graduate Research Fellowship (DGE-2236662).
\end{ack}

\bibliographystyle{unsrtnat}
\bibliography{bib}

% ============================================================
\appendix

% \section{Finite Element Exterior Calculus (FEEC) / Discrete Exterior Calculus (DEC) / Exterior Calculus?}
% \label{app:feec}
% \textcolor{red}{Like in the geo-new paper put a high level background section here which summarizes the previous papers, shows the de rham complex graphic.}

\section{Discrete exterior calculus background}
\label{app:dec_background}

This appendix is a short, self-contained primer on the discrete exterior calculus (DEC) constructions we use, following the graph-cochain conventions of \citet{trask2022enforcing}. The body of the paper only uses the $0$-/$1$-form portion of the complex. We describe the full de Rham complex here for context, and indicate at the end how the diagonal Hodge stars $M_0, M_1$ are obtained by a sparse linear solve.

\paragraph{The continuum de Rham complex.}
On a smooth domain $\Omega\subset\mathbb{R}^3$, the operators $\mathrm{grad}$, $\mathrm{curl}$, $\mathrm{div}$ form an exact sequence
\begin{equation*}
\begin{tikzcd}
C^\infty(\Omega) \arrow[r,"\mathrm{grad}"] & {[C^\infty(\Omega)]^3} \arrow[r,"\mathrm{curl}"] & {[C^\infty(\Omega)]^3} \arrow[r,"\mathrm{div}"] & C^\infty(\Omega),
\end{tikzcd}
\end{equation*}
encoded by the algebraic identities $\mathrm{curl}\circ\mathrm{grad}=0$ and $\mathrm{div}\circ\mathrm{curl}=0$. A discretization that reproduces this exact-sequence structure inherits vector calculus properties (e.g.\ $\mathrm{div}\,\mathrm{curl}=0$) as discrete algebraic identities rather than as numerical approximations.

\paragraph{Discrete cochains.}
Given an $\epsilon$-ball graph on a point cloud $\mathcal{X}$ with node set $V$, oriented edge set $E=\{(i,j):i<j\}$, and oriented triangle set $T$ (when used), we associate a real number with each graph entity. The resulting cochain spaces are
\[
C^0(V)\cong\mathbb{R}^{|V|},
\qquad
C^1(E)\cong\mathbb{R}^{|E|},
\qquad
C^2(T)\cong\mathbb{R}^{|T|}.
\]
A $0$-cochain $u\in C^0$ samples a scalar field at nodes; a $1$-cochain $\sigma\in C^1$ stores oriented line integrals $\sigma_e := \int_e v\cdot d\ell$; a $2$-cochain $\omega\in C^2$ stores oriented surface integrals. Cochains inherit orientation, so $\sigma_{ji}=-\sigma_{ij}$ and $\omega_{jik}=-\omega_{ijk}$.

\paragraph{Coboundary operators.}
The coboundary $d_k:C^k\to C^{k+1}$ is the discrete analog of the corresponding differential operator, given by signed combinatorial sums:
\begin{align*}
\text{(graph gradient)}\quad & (d_0 u)_{ij} \;=\; u_j - u_i, \\
\text{(graph curl)}\quad & (d_1 \sigma)_{ijk} \;=\; \sigma_{ij} + \sigma_{jk} + \sigma_{ki}.
\end{align*}
In matrix form $d_0\in\mathbb{R}^{|E|\times|V|}$ and $d_1\in\mathbb{R}^{|T|\times|E|}$ are the (signed) node-edge and edge-triangle incidence matrices. Note that these correspond to the generalized Stokes theorem $\int_\omega d\alpha = \int_{\partial\omega}\alpha$: for $d_0$ this is the fundamental theorem of calculus, $\int_{e_{ij}}\nabla u\cdot dl = u_j - u_i$; for $d_1$ this is Green's theorem, $\int_{T_{ijk}} \mathrm{curl}\,v\cdot dA = \int_{\partial T_{ijk}} v\cdot dl = \sigma_{ij} + \sigma_{jk} + \sigma_{ki}$, where $v$ is the vector field whose edge integrals are stored in the $1$-cochain $\sigma$. The chain-complex property $d_1\circ d_0 = 0$ is an algebraic identity of these incidence matrices; on each triangle $(i,j,k)$ the gradient values $u_j-u_i$, $u_k-u_j$, $u_i-u_k$ telescope to zero. These operators provide a discrete analog to $\mathrm{curl}\,\mathrm{grad}=0$. Together $V$, $E$, $T$ and the operators $d_0, d_1$ form the discrete cochain complex
\begin{equation*}
\begin{tikzcd}
C^0(V)\arrow[r,"d_0"] & C^1(E)\arrow[r,"d_1"] & C^2(T).
\end{tikzcd}
\end{equation*}

\paragraph{Hodge stars.}
The coboundary captures topology but not geometry, in the sense that the graph grad/curl are signed sums over incident edges/triangles, with no edge length or triangle area entering the expressions. Geometric information enters through diagonal positive-definite \emph{Hodge stars}
\[
M_0 = \operatorname{diag}(m_i),
\qquad
M_1 = \operatorname{diag}(a_e),
\qquad
M_2 = \operatorname{diag}(b_t),
\]
with positive entries playing the role of dual node volumes $m_i$, dual edge areas $a_e$, and dual triangle lengths $b_t$. The associated cochain inner products
$(\alpha,\beta)_{M_k} := \alpha^\top M_k\,\beta$
endow each $C^k$ with a metric, and we write $\|\cdot\|_{M_k}$ for the induced norm.

\paragraph{Codifferential as adjoint.}
The codifferential $\delta_{k+1}:C^{k+1}\to C^k$ is defined as the adjoint of $d_k$ with respect to the Hodge inner products,
\begin{equation*}
(d_k \alpha,\beta)_{M_{k+1}} \;=\; (\alpha,\delta_{k+1}\beta)_{M_k}
\qquad\text{for all $\alpha\in C^k$, $\beta\in C^{k+1}$,}
\end{equation*}
which yields the closed form
\begin{equation*}
\delta_{k+1} \;=\; M_k^{-1}\, d_k^\top\, M_{k+1}.
\end{equation*}
Specializing,
\[
\delta_1 = M_0^{-1} d_0^\top M_1 \quad\text{(graph divergence)},
\qquad
\delta_2 = M_1^{-1} d_1^\top M_2 \quad\text{(graph curl-adjoint)}.
\]
Because $d_1 d_0 = 0$ algebraically, the transpose identity $d_0^\top d_1^\top = 0$ gives $\delta_1\delta_2 = 0$, the discrete analog of $\mathrm{div}\,\mathrm{curl}=0$. The full primal/dual diagram is
\begin{equation*}
\begin{tikzcd}[column sep=large]
C^0(V) \arrow[r, shift left, "d_0"] & C^1(E) \arrow[l, shift left, "\delta_1"] \arrow[r, shift left, "d_1"] & C^2(T) \arrow[l, shift left, "\delta_2"].
\end{tikzcd}
\end{equation*}
Every discrete conservation law used in this paper is a corollary of this single algebraic structure. For any edge flux $\sigma\in C^1$, the full-vertex sum
\(
\mathbf{1}^\top d_0^\top M_1 \sigma = (d_0\mathbf{1})^\top M_1\sigma = 0
\)
is an algebraic identity since $d_0\mathbf{1}=0$. Under the boundary handling of Appendix~\ref{app:bcs} the residual equation is enforced only on interior nodes, and restricting the sum accordingly gives the form used in the body: for any $\sigma$ and $M_1$, the interior-summed residual equals the net boundary flux (and vanishes under zero-flux boundary data), expressing discrete global conservation. Likewise the Hodge Laplacian $\Delta_0 := \delta_1 d_0 = M_0^{-1} d_0^\top M_1 d_0$ is symmetric positive semidefinite in the $M_0$-inner product, with constant vectors as its null space. Therefore $\Delta_0 u = f$ has a unique solution up to constants when $f$ is mean-zero.

\paragraph{Solving for the Hodge stars.}
The diagonal entries of $M_0$ and $M_1$ are the only quantities in this construction not fixed by the graph topology, and they carry all the geometric information of the underlying point cloud. We do not assign them by hand: they are determined by a sparse linear KKT system that enforces local polynomial reproduction on $\delta_1$, derived in detail in Appendix~\ref{app:qp}. This is the meshfree analog of how DEC on a triangulation reads $m_i$, $a_e$ off the dual mesh. Here the dual mesh is replaced by the solution of a single quadratic program.

\begin{remark}\label{rem:shortcomplex}
To consider conservation laws of the form $\nabla \cdot \sigma = f$, we only require the first rung of the de Rham complex. The complete complex is needed for more general PDEs, and in particular for electromagnetism or fluid flow written in terms of rotational variables. While the KKT specification of Hodge stars could be similarly extended to the $2$-form case, the per-node moment system grows quickly with form degree, and the higher-form $B$ matrix is significantly less sparse, so we do not view this construction as computationally tractable for higher-order forms. For the PDEs considered in this paper, the $0$-/$1$-form complex suffices.
\end{remark}

\section{Quadratic program for the edge weights}
\label{app:qp}

We derive the linear KKT system that determines the edge weights $\{a_e\}$ entering the Hodge star $M_1$ in Section~\ref{sec:mfdec}. Recall the QP~\eqref{eq:opt} requires the codifferential $\delta_1$ to reproduce the negative divergence of every linear vector polynomial at every interior node. We translate this into linear equality constraints in $\{a_e\}$, write the KKT conditions, and reduce them to a single sparse Schur solve.

\paragraph{From polynomial reproduction to moment conditions.}
Fix an interior node $i\in\mathcal{X}_I$ and pick any $p\in(\mathbb{P}_1^d)^d$ with the basepoint expansion $p(x)=v+A(x-x_i)$, $v\in\mathbb{R}^d$, $A\in\mathbb{R}^{d\times d}$, so that $-\nabla\cdot p(x_i) = -\operatorname{tr}A$. The line integral over a canonically-oriented edge $e=(k,\ell)$ is exact for affine $p$ by midpoint quadrature,
\[
    \int_e p\cdot d\ell \;=\; \delta x_e\cdot p(\bar x_e).
\]
Substituting into $\delta_1 p = M_0^{-1}d_0^\top M_1 \sigma_p$ with $\sigma_p \in C^1(E)$ the cochain $(\sigma_p)_e := \int_e p\cdot d\ell$, and using $\eta_e := -s_e(i)\delta x_e = x_\ell - x_i$ (the displacement from $i$ to its neighbor across $e$) so that $\bar x_e - x_i = \tfrac12\eta_e$,
\[
    m_i\,(\delta_1 p)(x_i)
    \;=\; \sum_{e\ni i} s_e(i)\,a_e\,\delta x_e\cdot p(\bar x_e)
    \;=\; -\Big[\sum_{e\ni i}a_e\eta_e\Big]\!\cdot v
    \;-\; \tfrac12\,\operatorname{tr}\!\Big[A\sum_{e\ni i}a_e\eta_e\eta_e^\top\Big].
\]
Demanding $(\delta_1 p)(x_i) = -\operatorname{tr}A$ for arbitrary $(v,A)$ gives the per-node moment conditions
\begin{equation}
    \sum_{e\ni i} a_e\,\eta_e \;=\; 0,
    \qquad
    \sum_{e\ni i} a_e\,\eta_e\eta_e^\top \;=\; 2 m_i\, I_d,
    \label{eq:moments_app}
\end{equation}
or equivalently, since $\eta_e\eta_e^\top=\delta x_e\delta x_e^\top$,
\(
    \sum_{e\ni i} s_e(i)\,a_e\,\delta x_e = 0\) and \(\sum_{e\ni i} a_e\,\delta x_e\delta x_e^\top = 2 m_i I_d.
\)
This is $d$ first-moment scalar equations and $d(d+1)/2$ second-moment scalar equations per interior node, for a total of $n_c = (d+\tfrac{d(d+1)}{2})|\mathcal{X}_I|$ scalar constraints.

\paragraph{Linear system.} Stack \eqref{eq:moments_app} over interior nodes into the rectangular constraint matrix $B\in\mathbb{R}^{n_c\times|E|}$ and right-hand side $c\in\mathbb{R}^{n_c}$. Each row of $B$ is supported on the edges incident to a single node, so $B$ is sparse. The QP~\eqref{eq:opt} reads
\[
    \min_{a\in\mathbb{R}^{|E|}} \tfrac12\, a^\top \Phi^{-1} a
    \quad\text{s.t.}\quad
    Ba = c,
    \qquad \Phi := \operatorname{diag}(\phi_e),
\]
with $\phi_e = (1-r_e/\epsilon)^2_+$. Edges with $\phi_e=0$ are dropped from $E$ and play no role; on the retained edge set $\Phi^{-1}$ exists.

\paragraph{KKT conditions.} The Lagrangian
$\mathcal{L}(a,\lambda) = \tfrac12 a^\top\Phi^{-1}a - \lambda^\top(Ba-c)$
has stationarity in $a$, $\Phi^{-1}a - B^\top\lambda = 0$, and primal feasibility $Ba=c$, written jointly as the saddle-point system
\begin{equation}
    \begin{pmatrix} \Phi^{-1} & -B^\top \\ B & 0 \end{pmatrix}
    \begin{pmatrix} a \\ \lambda \end{pmatrix}
    \;=\;
    \begin{pmatrix} 0 \\ c \end{pmatrix}.
    \label{eq:kkt}
\end{equation}
Eliminating $a = \Phi B^\top\lambda$ from the first block gives the symmetric Schur system
\begin{equation}
    S\,\lambda \;=\; c, \qquad S := B\,\Phi\, B^\top \in \mathbb{R}^{n_c\times n_c}.
    \label{eq:schur}
\end{equation}
$S$ is symmetric positive semidefinite, and positive definite when $B$ has full row rank --- a generic property of regular point clouds with sufficient edge incidence ($n_c \ll |E|$). It inherits the sparsity of $B$: $S_{(i,\alpha),(j,\beta)}\ne 0$ only when nodes $i$ and $j$ share an edge, so $S$ is block-sparse with $O(|\mathcal{X}_I|\,d_{\rm avg}^2)$ nonzeros for average degree $d_{\rm avg}$. We solve \eqref{eq:schur} once per point cloud with a sparse direct factorization, then recover $a = \Phi B^\top\lambda$.

\paragraph{Implicit differentiation.} The map $\mathcal{X}\mapsto a(\mathcal{X})$ is required to be differentiable for end-to-end training. Differentiating \eqref{eq:schur},
\[
    S\,\partial\lambda \;=\; \partial c - (\partial S)\,\lambda,
    \qquad
    \partial a \;=\; (\partial\Phi)\,B^\top\lambda + \Phi\,(\partial B^\top)\,\lambda + \Phi\,B^\top\,\partial\lambda,
\]
which is implemented by autodiff through a custom adjoint of the sparse linear solve.

\paragraph{Optional positivity.} If $a_e\ge 0$ is required so that $M_1$ is a valid Hodge star, the QP gains inequality constraints,
\(
    \min_a \tfrac12 a^\top\Phi^{-1}a\) s.t. \(Ba = c,\; a\ge 0,
\)
with KKT conditions $\Phi^{-1}a - B^\top\lambda - \mu = 0$, $Ba=c$, $a\ge 0$, $\mu\ge 0$, $\mu_e a_e = 0$. We solve this with OSQP~\citep{osqp}; on the active set the equality-only Schur reduction above applies, and differentiability holds away from active-set transitions. We did not observe accuracy or stability differences between the equality-only and positivity-constrained solves in our experiments and use the equality-only form by default.

\begin{corollary}[Regularity of edge weights in nodal positions]
\label{cor:area-regularity}
Suppose the kernel takes the form $\phi(r) = (1 - r/\epsilon)^k_+$ for an integer $k \ge 1$. Then $\phi \in C^{k-1}([0,\infty))$ and the map $\mathcal{X} \mapsto a(\mathcal{X})$ is $C^{k-1}$ assuming invertible $S$. More generally, if the kernel is replaced by an arbitrary $\phi \in C^j$ supported on $[0, \epsilon]$, then $\mathcal{X} \mapsto a$ is $C^j$ under the same assumption.
\end{corollary}

\begin{proof}
The kernel-order claim follows from the general one applied to $j = k - 1$. For arbitrary $\phi \in C^j$, the diagonal weights $\Phi(\mathcal{X})$ are $C^j$ (composition of $\phi$ with the smooth distance map), the constraint matrix $B(\mathcal{X})$ has entries linear and quadratic in edge displacements (so $C^\infty$), and the right-hand side $c(\mathcal{X})$ depends on $m_i$ which is $C^j$ via the kernel-normalized volume rule. The Schur matrix $S = B \Phi B^\top$ is therefore $C^j$, and on the open set where it is non-singular, $\lambda = S^{-1} c$ is $C^j$ by the implicit function theorem applied to $S \lambda = c$. Hence, $a = \Phi B^\top \lambda$ is $C^j$. 
\end{proof}

\begin{remark}
The kernel $\phi(r) = (1 - r/\epsilon)^2_+$ used throughout this paper has $k = 2$, so Corollary~\ref{cor:area-regularity} gives $a \in C^1$ in the nodal positions. This $C^1$ regularity is what makes the end-to-end Jacobian $\partial a / \partial \mathcal{X}$, computed via the sparse-solve adjoint described above, well-defined as a continuous map.
\end{remark}

\section{Truncation error of the meshfree divergence}
\label{app:lpr_proof}

\begin{theorem}[\citealt{mirzaei2012generalized}, Theorem 4.3]
\label{thm:lpr}
Let $\Omega\subset\mathbb{R}^d$ be bounded and let $\Omega^*$ denote the closure of $\bigcup_{x\in\Omega}\mathcal{B}(x, C_2 h_0)$. If $\{\alpha_{ij}\}$ is a local polynomial reproduction of degree $m$ on $\Omega$ for $|\alpha|\le m$, then there exists a constant $C>0$ such that for every $u\in C^{m+1}(\Omega^*)$ and every point cloud $\mathcal{X}$ with $h\le h_0$,
\[
    \bigl|(D^\alpha u)(x_i) - D^\alpha_h u_i\bigr|
    \le C\, h^{m+1-|\alpha|}\, |u|_{C^{m+1}(\Omega^*)}.
\]
\end{theorem}

We give a simple Taylor-series derivation of the $O(h)$ truncation error for the meshfree codifferential $\delta_1$ when applied to a smooth vector field. This is a special case of Theorem~\ref{thm:lpr}. The divergence case admits a particularly transparent proof. The argument has two ingredients: (i) the QP makes $\delta_1$ exact on affine vector polynomials, which kills the constant and linear Taylor series terms of $\sigma$, and (ii) the Taylor remainder is bounded using a regularity assumption on the point cloud.

Let $\sigma:\Omega\to\mathbb{R}^d$ be $C^2$ with $M := \|\nabla^2\sigma\|_{L^\infty(\Omega)}$. Fix an interior node $x_i$, and recall that the cochain associated to $\sigma$ is
\(
\sigma_e := \int_e\sigma\cdot d\ell,
\)
so that the meshfree codifferential reads
\[
(\delta_1\sigma)(x_i) = \frac{1}{m_i}\sum_{e\ni i}s_e(i)\,a_e\,\sigma_e.
\]
We assume $\epsilon$ to be chosen proportional to $h$ and that the point cloud is quasi-uniform: edge lengths satisfy $c h\le\|\delta x_e\|\le \epsilon = O(h)$ on every retained edge, and node degrees are bounded. Together with the second-moment identity $\sum_{e\ni i}a_e\|\delta x_e\|^2 = 2dm_i$ from Appendix~\ref{app:qp}, this gives the stencil-stability bound
\begin{equation}
\frac{1}{m_i}\sum_{e\ni i}a_e\,\|\delta x_e\| \;\le\; C_{\rm reg}\,h^{-1},
\label{eq:weight_bound}
\end{equation}
which is the analog of Mirzaei's boundedness condition in the present cochain setting.

Write the affine Taylor expansion of $\sigma$ at $x_i$,
\[
\sigma(x) \;=\; p(x) + r(x),
\qquad
p(x) := \sigma(x_i) + \nabla\sigma(x_i)(x-x_i),
\]
and observe two consequences. First, $p$ is affine, so $\nabla\cdot p$ is a constant equal to $\nabla\cdot\sigma(x_i)$. Second, the Taylor remainder satisfies
\(
|r(x)|\le \tfrac12\|x-x_i\|^2 M.
\)
The cochain of $\sigma$ splits accordingly as $\sigma_e = p_e + r_e$ with $p_e := \int_e p\cdot d\ell$ and $r_e := \int_e r\cdot d\ell$, and writing $\delta_1 p$, $\delta_1 r$ for $\delta_1$ applied to those two cochains, linearity gives
\begin{equation*}
(\delta_1\sigma)(x_i) \;-\; \bigl(-\nabla\cdot\sigma\bigr)(x_i)
\;=\; \underbrace{\bigl[(\delta_1 p)(x_i) + \nabla\cdot p(x_i)\bigr]}_{\text{polynomial term}}
\;+\; \underbrace{(\delta_1 r)(x_i)}_{\text{remainder term}}.
\end{equation*}

The QP~\eqref{eq:opt} is constructed precisely so that $\delta_1$ reproduces $-\nabla\cdot$ on affine vector fields. Concretely, the moment conditions~\eqref{eq:moments_app} give, for any $p\in{(\mathbb{P}_1^d)}^d$,
\[
(\delta_1 p)(x_i) = -\nabla\cdot p(x_i),
\]
so the polynomial bracket above is identically zero. This is the only place the LPR property is used in the argument.

For any edge $e$ incident to $x_i$ and any $x\in e$, we have $\|x-x_i\|\le \epsilon$, so $|r(x)|\le \tfrac12\epsilon^2 M$ uniformly along the edge. Hence the cochain value satisfies
\[
|r_e| \;=\; \left|\int_e r\cdot d\ell\right|
\;\le\; \|\delta x_e\|\cdot \tfrac12\epsilon^2 M.
\]
Substituting into the codifferential and using $|s_e(i)|=1$,
\[
\bigl|(\delta_1 r)(x_i)\bigr|
\;\le\; \frac{1}{m_i}\sum_{e\ni i}a_e\,|r_e|
\;\le\; \tfrac12\epsilon^2 M \cdot \frac{1}{m_i}\sum_{e\ni i}a_e\,\|\delta x_e\|
\;\stackrel{\eqref{eq:weight_bound}}{\le}\;
\tfrac12\,C_{\rm reg}\,M\,\epsilon^2 h^{-1}
\;=\; O(h)\,M.
\]

Combining,
\begin{equation}
\bigl|(\delta_1\sigma)(x_i) - (-\nabla\cdot\sigma)(x_i)\bigr|
\;\le\; \tfrac12\,C_{\rm reg}\,\|\nabla^2\sigma\|_{L^\infty}\,h.
\label{eq:div_truncation}
\end{equation}
Hence $\delta_1$ has $O(h)$ pointwise truncation error on $C^2$ fields, which recovers the $|\alpha|=1, m=1$ case of Theorem~\ref{thm:lpr}. The argument extends to higher-order polynomial reproduction by including more terms in the Taylor expansion: each additional moment that the QP matches kills one more Taylor series term, and the remainder gains one more power of $h$. 

\begin{remark}
\textbf{Truncation error vs. solution error.} We stress that this error analysis provides a bound on the truncation error of the discrete operator $\delta_1$ when applied to a smooth vector field, not a bound on the solution error of the PDE. The latter would require a complete Lax-Milgram/Cea's lemma-style argument to establish which is beyond the scope of the present paper. Throughout the results, we observe when \textit{solving} the learned PDE that the solution error scales as $O(h^2)$, and additional order of convergence than that predicted by the truncation error analysis.
\end{remark}

\section{Stability, consistency, and generalization error}
\label{app:stability_consistency}

We now make precise how an error in the learned local constitutive law propagates to an error in the global PDE solution. The argument follows the standard principle that consistency and stability together imply convergence.
Unlike solution operator methods, which make direct predictions, our learned model produces a local edge flux within a conservative discrete PDE solve. Thus a local flux-density error induces a residual perturbation, and stability of the implicit solve converts this residual perturbation into a solution error.

For clarity, consider target continuum equations of the form
\[
    -\epsilon \Delta u^\star + \nabla\cdot \sigma^\star(u^\star,\nabla u^\star;\mu)=f
    \quad \text{in } \Omega,
    \qquad
    u^\star=g \quad \text{on } \partial\Omega_D,
\]
where \(\epsilon>0\) is the background diffusion and \(\sigma^\star\) is the nonlinear constitutive flux. This matches the diffusion-plus-flux structure of the learned model. More general flux laws may be representable by the learned term, but if the learned flux cancels the coercive diffusion, the simple stability argument below no longer applies directly.

Let \(I_hu^\star\) denote nodal sampling of the continuum solution. After enforcing Dirichlet boundary conditions, set \(K:=\epsilon d_0^\top M_1d_0\), \(D:=d_0^\top M_1\), and \(b:=M_0f\). The learned residual is
\[
    G_\theta(u):=Ku+DF_\theta(u;\mu)-b,
    \qquad
    G_\theta(u_\theta)=0.
\]
We also define the ideal discrete flux obtained by evaluating the exact constitutive kernel with the same edge features and edge integration rule as the learned model:
\[
    F^\star_{h,e}(I_hu^\star;\mu):=r_e\mathcal K^\star(\xi_e(I_hu^\star,\mu)).
\]
This object is used only in the analysis.

\begin{assumption}[Stability of the learned solve]
\label{assump:stability}
Let the nonlinear residual contribution \(\mathcal F_\theta(u;\mu):=DF_\theta(u;\mu)\) be Lipschitz with constant \(C_\theta\).  We assume \(\tau:=\|K^{-1}\|C_\theta<1\).
\end{assumption}
This is a standard assumption that the Laplacian provides sufficient control over the nonlinearity (see example, \citep{trask2022enforcing}). We define
\[
    C_{\mathrm{stab}}(h):=\frac{\|K^{-1}\|}{1-\tau},
    \]
including a possible $h$-dependence due to the fact that $\|K^{-1}\|$ is likely to depend on a discretization lengthscale. For comparison, in nodal P1 FEM the matrix norm scales inverrsely with mesh diameter. A tight analysis in the MEEC setting would require an extension of the classical Lax-Milgram analysis for continuous Galerkin methods, which we defer to future work. 

\begin{assumption}[Discretization consistency]
\label{assump:ideal_consistency}
Following the truncation error analysis in Appendix~\ref{app:lpr_proof}, the exact continuum solution sampled on the point cloud satisfies
\[
    \left\|KI_hu^\star + DF_h^\star(I_hu^\star;\mu)-b\right\|
    \le C_{\mathrm{disc}}h,
\]
following from the LPR property.
\end{assumption}

Next, we define the constitutive model error
\[
    \gamma_\epsilon:=\sup_{\xi}\left|\mathcal K_\theta(\xi)-\mathcal K^\star(\xi)\right|.
\]

\begin{lemma}[Discrete divergence of flux error]
\label{lem:flux_to_residual}
Consider any edge cochain \(E\in C^1\) associated with a twice differentiable field. From a Taylor series argument, one can show the scaling \(|E_e|\le r_e\eta\), for some $\eta \in\mathbb{R}$ and from the polynomial reproduction property that
\[
    \|DE\|\le C_{\mathrm{flux}}\eta,
\]
with \(C_{\mathrm{flux}}\) independent of the particular geometry and resolution. Informally, this says that a pointwise edge-flux error of size $\eta$ inflates to a nodal residual of size at most $C_{\mathrm{flux}}\eta$ after the discrete divergence, uniformly over the admissible family. Then, from the definition of the constitutive model error,
\[
    \left\|D\left[F_\theta(I_hu^\star;\mu)-F_h^\star(I_hu^\star;\mu)\right]\right\|
    \le C_{\mathrm{flux}}\gamma_\epsilon.
\]
\end{lemma}

\begin{proof}
Let \(E:=F_\theta(I_hu^\star;\mu)-F_h^\star(I_hu^\star;\mu)\). Since both fluxes use the same edge features and the same edge integration rule,
\[
    |E_e|
    =r_e\left|\mathcal K_\theta(\xi_e)-\mathcal K^\star(\xi_e)\right|
    \le r_e\gamma_\epsilon.
\]
The assumed boundedness of the discrete divergence gives \(\|DE\|\le C_{\mathrm{flux}}\gamma_\epsilon\).
\end{proof}

\begin{theorem}[Error decomposition]
\label{thm:generalization_error}
Under Assumptions~\ref{assump:stability}, \ref{assump:ideal_consistency} and Lemma~\ref{lem:flux_to_residual}, the learned solution satisfies
\[
    \|u_\theta-I_hu^\star\|
    \le
    C_{\mathrm{stab}}(h)
    \left(
        C_{\mathrm{disc}}\,h
        +
        C_{\mathrm{flux}}\gamma_\epsilon
    \right).
\]
\end{theorem}

\begin{proof}
First, for any \(v\) satisfying the same Dirichlet conditions as \(u_\theta\), from \(G_\theta(u_\theta)=0\), we can show stability in the residual:

\[
    K(u_\theta-v)
    =-G_\theta(v)-D\left[F_\theta(u_\theta;\mu)-F_\theta(v;\mu)\right].
\]
Therefore,
\[
    \|u_\theta-v\|
    \le \|K^{-1}\|\left(\|G_\theta(v)\|+C_\theta\|u_\theta-v\|\right),
\]
and using \(\tau=\|K^{-1}\|C_\theta<1\),
\[
    \|u_\theta-v\|\le C_{\mathrm{stab}}\|G_\theta(v)\|.
\]
Taking \(v=I_hu^\star\) gives \(\|u_\theta-I_hu^\star\|\le C_{\mathrm{stab}}\|G_\theta(I_hu^\star)\|\).

It remains to bound this residual. Add and subtract the ideal discrete flux term:
\[
    G_\theta(I_hu^\star)
    =
    \left[KI_hu^\star+DF_h^\star(I_hu^\star;\mu)-b\right]
    +
    D\left[F_\theta(I_hu^\star;\mu)-F_h^\star(I_hu^\star;\mu)\right].
\]
By the triangle inequality, Assumption~\ref{assump:ideal_consistency}, and Lemma~\ref{lem:flux_to_residual},
\[
    \|G_\theta(I_hu^\star)\|
    \le
    C_{\mathrm{disc}}\,h
    +
    C_{\mathrm{flux}}\gamma_\epsilon.
\]
Combining this with residual stability proves the result.
\end{proof}

\begin{remark}[Meaning of the bound]
This bound serves to disentangle the solution error into two contributions. In the first term on the left, we obtain aconventional error term balancing numerical stability against truncation error. This corresponds to the celebrated Lax Equivalence Theorem \citep{LaxRichtmyer1956}, which in the current setting manifests as the implication that \textit{if} for uniformly bounded $C_{\mathrm{stab}}$, the $O(h)$ truncation error implies convergence to the solution as $h \rightarrow 0$. The second term on the right however identifies a new source of error in the model class of the flux network. If a network of sufficient capacity can be sufficiently trained, then we recover the true solution as $h \rightarrow 0$. These two terms loosely correspond to the responsibilities of MEEC and MEEC-Net in handling global and local aspects of the error, respectively, providing a universal approximation result. We caution, however, that this bound should not be interpreted as a tight estimation of convergence rates, as accounting for the interplay between stability, truncation error, and variational crimes, likely by generalizing the Lax-Milgram analysis and Strang's lemma to account for the meshfree approximation to a continuous operator.

\end{remark}
\section{Boundary conditions in meshless DEC}
\label{app:bcs}

\paragraph{Dirichlet.}
Let $\partial\Omega_D \subset \mathcal{X}$ denote the set of nodes carrying prescribed values $g_i$.
We set $m_i = 0$ for all $i \in \partial\Omega_D$ so that boundary rows are excluded from the linear system, then overwrite $u_i \leftarrow g_i$ after each solve or network forward pass.
% This is the standard row-elimination approach wherein the zero mass weight ensures boundary nodes contribute no divergence residual.

\paragraph{Neumann.}
Natural (flux) boundary conditions $\mathcal{F}\cdot\hat{n} = q$ on $\partial\Omega_N$ are incorporated by modifying the discrete divergence at boundary nodes.
For each $i \in \partial\Omega_N$ we estimate the outward normal $\hat{n}_i$ from the local point distribution, then restrict the edge sum in $(\delta_1\sigma)_i$ to edges whose midpoint projects onto the interior half-space ($\delta x_e \cdot \hat{n}_i < 0$), and add the Neumann flux contribution $q\,m_i$ to the right-hand side.
Homogeneous Neumann conditions ($q=0$) are the natural condition of the variational formulation and are satisfied automatically when no flux term is added (requiring no modification).

\section{Edge-frame projection}
\label{app:pullback}
We define the network inputs in terms of invariant edge-frame projected quantitities. The same projection procedure is applied to states and additional physical inputs, either of which may be any combination of scalar, vector, or tensor valued.
Using the edge geometry $(\delta x_e, r_e, \bar x_e, \hat e_e)$ from Section~\ref{sec:mfdec}, in 2D let \(\hat n_e=R_{\pi/2}\hat e_e\) and \(B_e=(\hat e_e,\hat n_e)\). For a nodal channel \(a\), set \(\bar a_e=(a_i+a_j)/2\) and \(D_e a=(a_j-a_i)/r_e\). The projection \(\Pi_e\) maps \((\bar a_e,D_e a)\) to scalar edge-frame components according to the type of \(a\).
For a scalar \(a\in\R\) or vector \(b\in\R^2\),
\[
\Pi_e(a) = \left( \bar a_e, D_e a \right), \quad \textrm{and} \quad\Pi_e(b) = \left( \bar b_e\cdot \hat e_e, \bar b_e\cdot \hat n_e, D_e b\cdot \hat e_e, D_e b\cdot \hat n_e \right).
\]
This could be similarly defined for tensor inputs in $\R^{d \times d}$, but was not needed in this work.
% For a $d=3$ vector channel we fix the two normal directions by \textcolor{red}{... FINISH} 
The full edge feature vector is obtained by stacking these projections over all state and conditioning channels \(\xi_e = \operatorname{stack}_{a\in u,\mu}\Pi_e(a).\)

If the underlying fields are \(C^2\) on a neighborhood of \(e\), then \(\bar a_e\) and \(D_e a\) are second-order midpoint approximations of \(a(\bar x_e)\) and \((\nabla a)(\bar x_e)\hat e_e\). Therefore \(\Pi_e(a)\) is a second-order approximation of the corresponding edge-frame components of \(a\) and its edge-directional derivative. In particular, \(\xi_e\) consistently approximates the local first-order variables \((u,\nabla u)\) expressed in the oriented edge frame.

\section{Architecture and training}
\label{app:arch}

All models are trained with the SOAP optimizer~\citep{vyas2024soap} with momentum parameters $\beta_1 = 0.95$, $\beta_2 = 0.99$, and no weight decay.
Gradients are clipped to unit norm before each parameter update.
We use a batch size of one for training all methods, with random re-sampling from the dataset at each iteration.
The model checkpoint achieving the lowest training error is used for evaluation.
For our method (MEEC-Net) and each baseline, the same training budget (iterations and learning rate) is applied uniformly. Any method-specific hyperparameters (e.g.\ number of attention slices or latent dimension) are set according to reasonable defaults from the published values.
The MEEC architecture is a feedforward MLP with Tanh activations applied to the local edge features. In all experiments we choose $\epsilon$ for the neighbor graph such that the minimum number of neighbors over the graph exceeds some threshold ($8$ in $d=2$ and $12$ in $d=3$), which in a quasi-uniform point cloud provides a proportional $\epsilon$ to $h$.
All models were trained on a single H200.

The QP and solves in the Newton iteration are handled on CPU with SciPy sparse LU solves, with a PyTorch wrapper for differentiability. We use Newton line search, with a fixed tolerance for all examples ($1e-6$). We use Picard iteration as a fallback.

Table~\ref{tab:training_hparams} summarizes the per-experiment configuration.
The hidden dimension and depth refer to the MEEC-Net MLP. The baselines use the same hidden dimension when relevant.

\begin{table}[h]
\centering
\caption{Training hyperparameters per experiment. $N_{\rm train}$ denotes the pool size used in the main data-efficiency sweep and $h$ is the target mesh spacing.}
\label{tab:training_hparams}
\setlength{\tabcolsep}{6pt}
\begin{tabular}{lcccccc}
\toprule
\textbf{Experiment} & \textbf{Hidden} & \textbf{Depth} & \textbf{Iterations} & \textbf{LR} & $N_{\rm train}$ (max) & $h$ \\
\midrule
A.D. (both) & 64  & 4 & 5\,000  & $5\times10^{-4}$ & 100 & 0.07 \\
Darcy (both) & 64  & 4 & 5\,000  & $5\times10^{-4}$ & 100 & 0.07 \\
Elasticity & 64  & 4 & 5\,000  & $5\times10^{-4}$ & 100 & 0.07 \\
SimJEB & 128  & 4 & 10\,000  & $5\times10^{-4}$ & 50 & mesh \\
\bottomrule
\end{tabular}
\end{table}

\section{Additional results}
\label{app:results}

Table~\ref{tab:mlp_ablation} reports ID and OOD test relative $L_2$ error as a function of hidden dimension and depth ($N_{\rm train}=100$, 5\,000 iterations).
ID error saturates around $64\times4$, while OOD error is slightly lower for smaller architectures, so $64\times4$ is used throughout as a conservative default.

\begin{table}[h]
\centering
\caption{MEEC MLP ablation on advection-diffusion for the experimental described in~\ref{app:exp_setup}, $N_{\rm train}=100$, 5\,000 iterations. Values are Relative $L_2$ (1e-2).}
\label{tab:mlp_ablation}
\setlength{\tabcolsep}{5pt}
\begin{tabular}{lcccc|cccc}
\toprule
 & \multicolumn{4}{c|}{\textbf{ID test relative $L_2$ error (1e-2)}} & \multicolumn{4}{c}{\textbf{OOD test Rel. $L_2$ error (1e-2)}} \\
\textbf{Hidden} & $L{=}1$ & $L{=}2$ & $L{=}3$ & $L{=}4$ & $L{=}1$ & $L{=}2$ & $L{=}3$ & $L{=}4$ \\
\midrule
16  & 1.01 & 0.62 & 0.65 & 0.59 & 2.02 & 1.55 & 1.38 & 1.39 \\
32  & 0.86 & 0.65 & 0.57 & 0.52 & 1.61 & 1.66 & 1.73 & 1.47 \\
64  & 0.74 & 0.57 & 0.51 & \textbf{0.49} & \textbf{1.37} & 1.63 & 1.70 & 1.63 \\
128 & 0.71 & 0.59 & 0.50 & 0.51 & 1.48 & 1.78 & 1.70 & 1.58 \\
\bottomrule
\end{tabular}
\end{table}

\subsection{Resolution scaling and convergence}

Table~\ref{tab:resolution} reports test relative $L_2$ error versus mesh spacing $h$ for a fixed $64\times4$ model trained on $N_{\rm train}=100$ samples for 20\,000 iterations. We empirically verify a $\approx 1.91$ convergence rate, consistent with the solution-level rate observed on Poisson (Section~\ref{sec:mfdec}).

\begin{table}[h]
\centering
\caption{MEEC resolution scaling on advection--diffusion (OBSTACLES). Convergence rate $\approx 1.91$.}
\label{tab:resolution}
\begin{tabular}{cccc}
\toprule
$h$ & $N$ (approx.) & Final ID Rel. $L_2$ error \\
\midrule
0.100 &    544 & 0.96 \\
0.070 &  1\,071 & 0.54 \\
0.050 &  1\,963 & 0.24 \\
0.035 &  3\,966 & 0.13 \\
\bottomrule
\end{tabular}
\end{table}

\subsection{SO(d) invariance}
We demonstrate SO($d$) invariance of our approach in Figure~\ref{fig:sod_invariance}.

\begin{figure}[!t]
    \centering
    \includegraphics[width=0.75\linewidth]{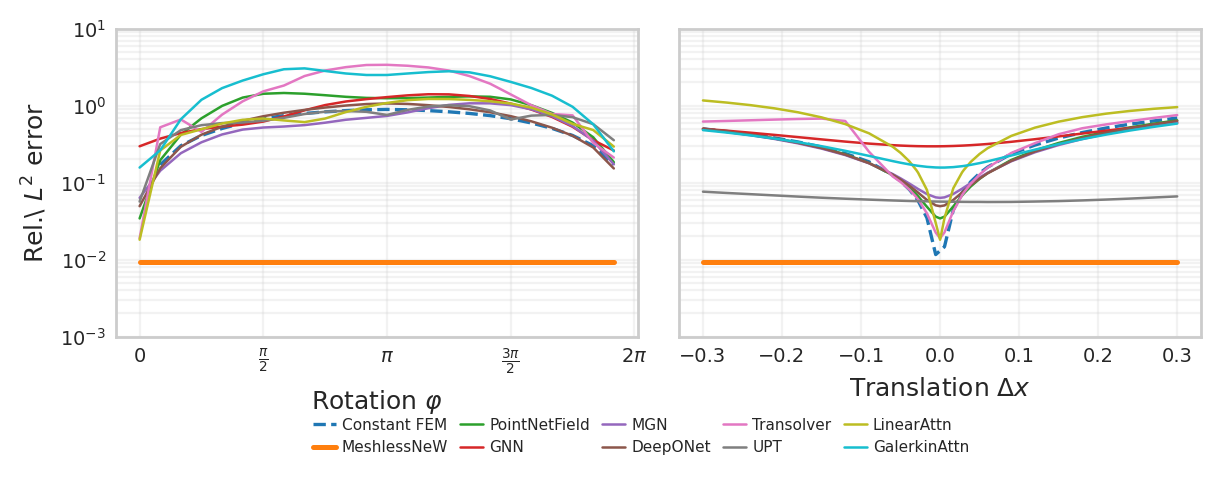}
    \caption{The model inputs are purely in a SO($d$) invariant edge-projected space, which results in invariance in the solutions resulting from solving the model.}
    \label{fig:sod_invariance}
\end{figure}

\subsection{OOD data efficiency}
In Figure~\ref{fig:ood_data_efficiency} we report the relative $L_2$ error on the OOD test data as a function of $N_{\text{train}}$ for the advection diffusion example in the main text. In addition to the geometric variations, this experiment includes parameter variations in $\mu$ as a function of the constant velocity angle, $\theta$. The OOD test $\theta$ range is outside the training conditions and therefore requires some degree of physical extrapolation, which can not be achieved with any of the baselines considered, irrespective of training density. MEEC-Net, however, provides consistent and highly accurate predictions in this regime from as few as a single sample.

\begin{figure}
    \centering
    \includegraphics[width=0.55\linewidth]{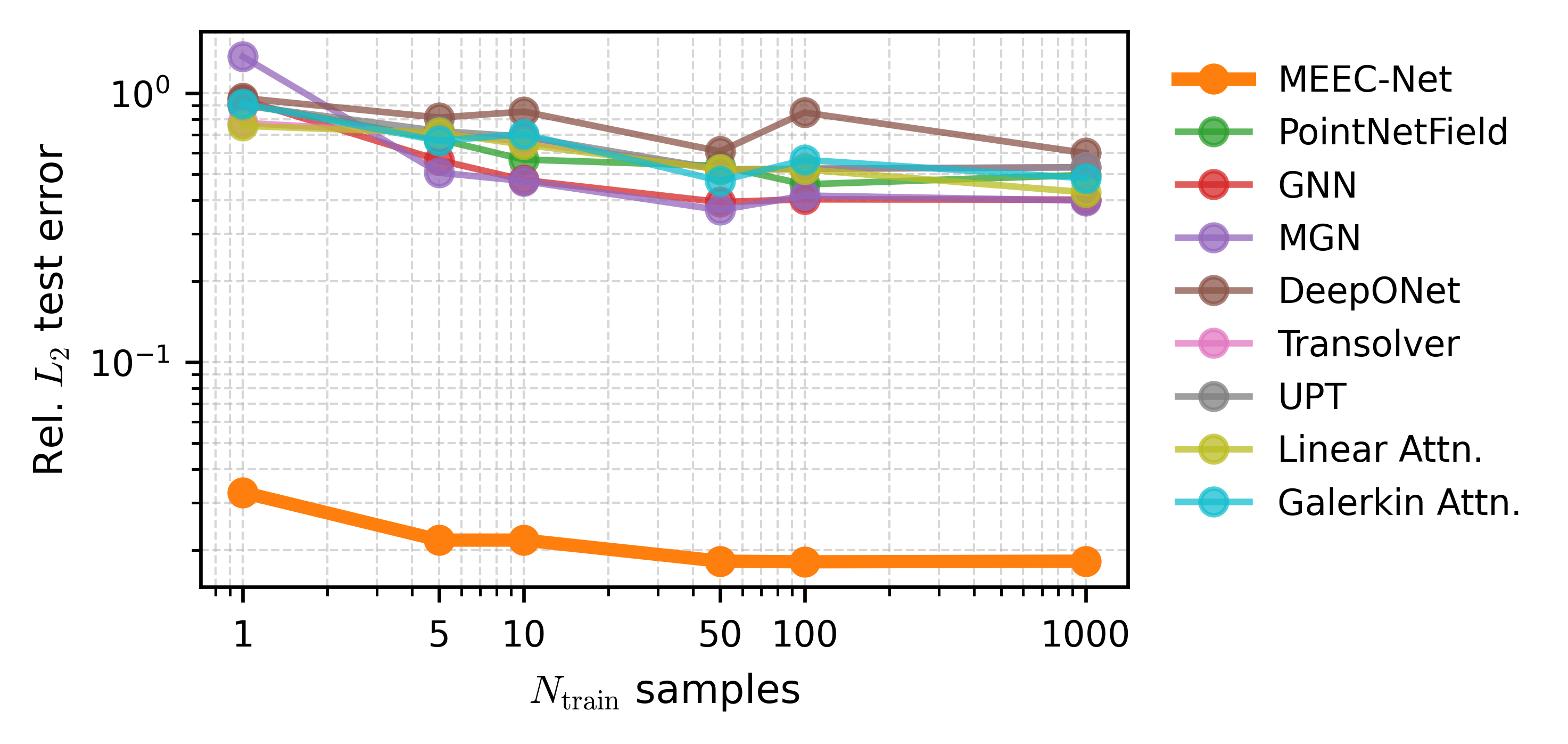}
    \caption{OOD data efficiency for advection diffusion experiment, corresponding to Figure~\ref{fig:data_efficiency}}
    \label{fig:ood_data_efficiency}
\end{figure}

\subsection{Edge feature distribution}
In Figure~\ref{fig:edge_feature_KDE} we consider the 1D distribution of edge space invariant features used in the MEEC-Net formulation. Despite variations in geometry (with two occlusions in the OOD case) and physics via the velocity field, the local frame features provide a consistent representation, which enables the strong generalization of the MEEC-Net approach.

\begin{figure}
    \centering
    \includegraphics[width=0.75\linewidth]{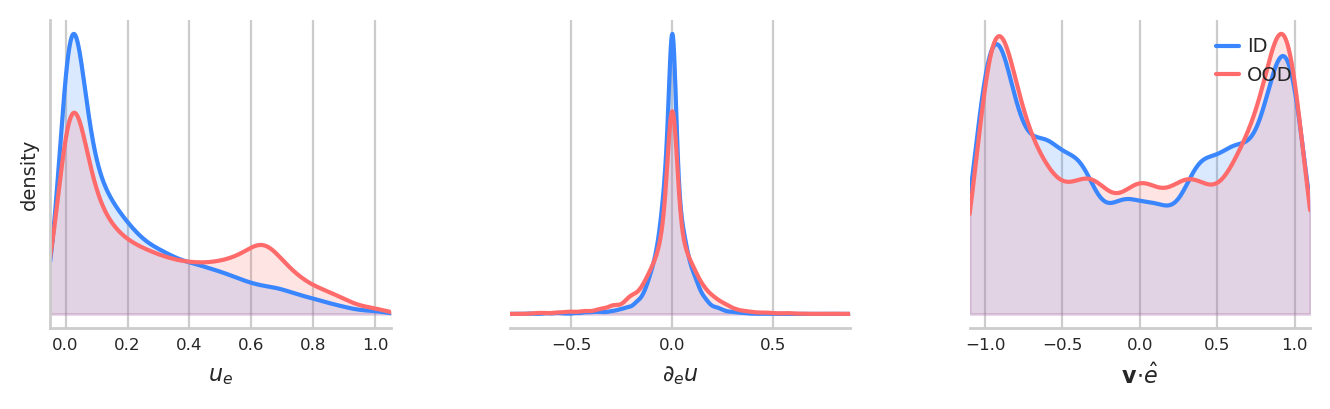}
    \caption{We show the edge feature distributions over a single sample ID and OOD problem for the advection diffusion problem, showing a consistent distribution under geometry and physics shift (via velocity direction), which enables generalization in the learned model.}
    \label{fig:edge_feature_KDE}
\end{figure}

\subsection{Timing comparisons}
We report timing comparisons for all baselines trained on the advection-diffusion experiment reported in Figure~\ref{fig:sweeps_v2}. Since the MEEC-Net approach involves a nonlinear solve around the edge network evaluation for each prediction, we find it is $10$-$100\times$ slower than the direct prediction approaches. On the canonical benchmark resolution used in \ref{fig:sweeps_v2}, MEEC-Net requires \(1.78\times 10^{-2}\) seconds per sample, compared with \(1.94\times 10^{-4}\)--\(1.83\times 10^{-3}\) seconds for baselines. While this can be seen as a limitation of our approach, we emphasize our method offers substantial advantages as a data-efficient and structure-preserving model, beyond efficient inference, namely the large generalization of physics, geometry, and BCs. In cases where a large training corpus can be efficiently curated, direct prediction approaches may be preferable. But in complex engineering tasks where data collection is typically limited and expensive, we view additional structure as a necessary component for providing a useful model.

\begin{table}[t]
\centering
\caption{Mean inference time per sample on the canonical benchmark resolution. MEEC-Net is slower because prediction requires a nonlinear solve, while baselines are direct forward models.}
\label{tab:timing}
\begin{tabular}{lc}
\toprule
Method & Time / sample (s) \\
\midrule
MEEC-Net & \(1.783\times 10^{-2}\) \\
PointNetField & \(1.938\times 10^{-4}\) \\
GNN & \(8.858\times 10^{-4}\) \\
MGN & \(9.214\times 10^{-4}\) \\
DeepONet & \(3.035\times 10^{-4}\) \\
Transolver & \(1.375\times 10^{-3}\) \\
UPT & \(1.833\times 10^{-3}\) \\
LinearAttentionNet & \(1.033\times 10^{-3}\) \\
GalerkinAttentionNet & \(8.888\times 10^{-4}\) \\
\bottomrule
\end{tabular}
\end{table}

\section{Experiment description}\label{app:exp_setup}

\subsection{Canonical PDEs}
\paragraph{Geometries}
For the experiments on canonical PDEs we use a standardized 2D geometry consisting of cutouts from a unit square to enable standardized comparison. We consider two distributions: the first for training and in-distribution (ID) testing, the second with larger variation for out-of-distribution (OOD) testing. Sample geometries and solutions fields for the advection diffusion system are shown in Figure~\ref{fig:id_ood_geoms_canonical}. Specifically, for the ID generation, we use a constant radius obstacle with the center location varying along $y=-x$ diagonal, for the OOD geometries we consider 1 or 2 obstacles with variable radius, placed anywhere in the domain (non-intersecting the outer boundaries). For each experiment in Table~\ref{tab:canonical_results} we additionally consider some source of physical parameter perturbation for the OOD experiments (in $\mu$ or $u_b$), described below for each PDE.

\begin{figure}[!t]
    \centering
    \includegraphics[width=0.75\linewidth]{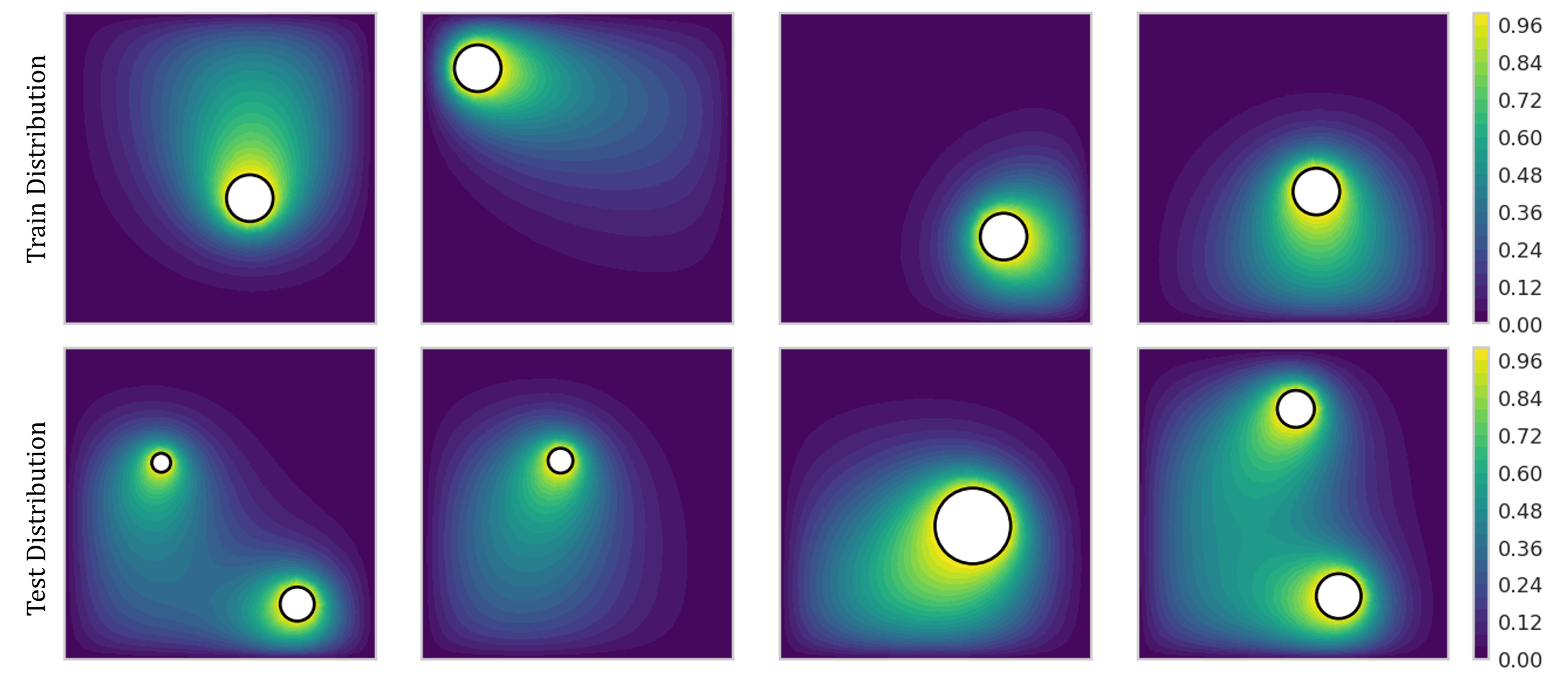}
    \caption{Training and in-distribution geometries (top row) and out-of-distribution test geometries (bottom) row, with solution fields from advection diffusion.}
    \label{fig:id_ood_geoms_canonical}
\end{figure}
\paragraph{Advection--diffusion}
The state is a scalar concentration field $u$.
The PDE is
\[
    -\epsilon\,\Delta u + \mathbf{v} \cdot \nabla u = 0,
    \qquad \mathbf{v}(\theta) = (\cos\theta,\,\sin\theta),
\]
with fixed diffusivity $\epsilon = 0.2$.
Dirichlet conditions impose $u = 1$ on all hole boundaries and $u = 0$ on the outer square.
We pass the velocity field as a vector input tp the models $\mu=\mathbf{v}\in \mathbb{R}^{N \times 2}$.
The flux is $\mathcal{F} = \epsilon\,\nabla u - \mathbf{b}_\theta u$.
For ID data we consider $\theta\in[270^\circ,220^\circ)$ for test, $\theta\in[220^\circ,270^\circ)$

The nonlinear variant replaces the constant advection coefficient by a concentration-dependent one:
\[
    -\epsilon\,\Delta u + (1 + u)\,\mathbf{v}\cdot \nabla u = 0,
\]
with identical boundary conditions and inputs.

\paragraph{Darcy flow}
The state is a scalar field $u$ (e.g. pressure or concentration).
The PDE is
\[
    -\nabla \cdot \bigl(\alpha(\mathbf{x})\,\nabla u\bigr) = 1,
\]
with homogeneous Dirichlet conditions $u = 0$ on all boundaries (hole and outer square).
The permeability $\alpha(\mathbf{x}) > 0$ is a random log-normal field, generated as a truncated Fourier sum with spatially variable amplitude. We define the amplitude $\kappa \in \operatorname{lognormal}(0.0,0.5)$ for ID data generation and $\kappa \in \operatorname{lognormal}(0.0,0.7)$ for OOD, and use 8 spatial fourier modes for the basis for each.
The permeability is provided as the model input as $\mu = \alpha \in \R^N$. 
The flux is $\mathcal{F} = -\alpha\,\nabla u$.

The nonlinear variant uses a field-dependent permeability:
\[
    -\nabla \cdot \bigl((\alpha(\mathbf{x}) + u)\,\nabla u\bigr) = 1,
\]
with the same boundary conditions and conditioning.

\paragraph{Linear elasticity}
The state is a 2D displacement field $\mathbf{u} = (u_x, u_y)$.
The PDE is the plane-strain linear elasticity system $-\nabla \cdot \boldsymbol{\sigma}(\mathbf{u}) = \mathbf{0}$,
where $\boldsymbol{\sigma} = 2\mu\,\boldsymbol{\varepsilon}(\mathbf{u}) + \lambda\,(\nabla\cdot\mathbf{u})\,\mathbf{I}$.
The left boundary is clamped ($\mathbf{u} = \mathbf{0}$); the right boundary is pulled with a fixed-amplitude displacement
$\mathbf{d}(\theta) = \delta(\cos\theta,\,\sin\theta)$ with $\delta = 0.1$.
The hole boundary and the top and bottom edges are traction-free. We vary only the prescribed displacement via the boundary conditions in $u_b$ in this example, and do not have a field input. For ID data generation we consider $\theta\in[270^\circ,220^\circ)$ and for OOD we instead use $\theta\in[220^\circ,270^\circ)$.

\subsection{SimJEB: simulated jet engine bracket}
The Simulated Jet Engine Bracket Dataset (SimJEB) \citep{whalen2021simjeb} is a open repository of submissions to a GE Jet Engine Bracket Challenge which have been uniformly processed for learning and related tasks. There are 381 highly diverse designs, each of which conform to certain constraints to match rivet and load placements. To form the dataset, each design was subjected to the same horizontal, vertical, diagonal, or torsional loading configurations, in a FEM linear elasticity solver. In our experiments we consider only the vertical loading configuration, and use an equivalent \emph{prescribed-displacement} problem specification via dirichlet BCs on the loading points for all models. The point clouds are sub-sampled to $4096$ points. We train on $1$ to $50$ geometries and test on $20$ held-out samples 

\begin{figure}[!t]
    \centering
    \includegraphics[width=0.75\linewidth]{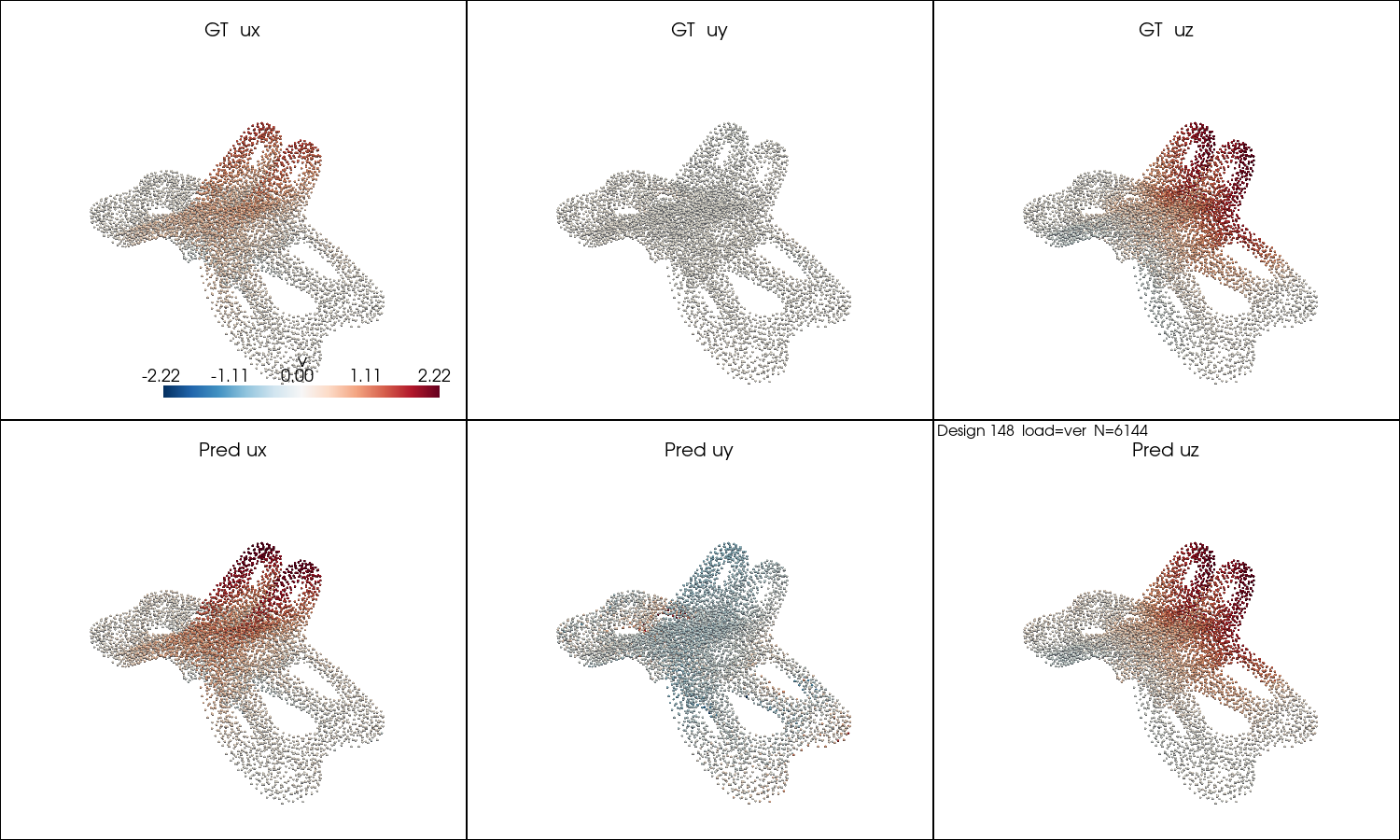}
    \caption{Qualitative prediction comparison on a representative SimJEB validation geometry from MEEC-Net.}
    \label{fig:simjeb_qualitative_v1}
\end{figure}

\section{Baseline methods}
\label{app:baselines}

All baselines receive the same per-node input features and are trained with the same optimizer, learning rate, iteration budget, and loss as MEEC.
Let $N_F$ be the total number of field \emph{components}, i.e. flattening the vector or tensor valued fields, likewise let $N_P$ be the number of physical parameter components with the same flattening.
The per-node feature vector at node $i$ is
\begin{equation}
    \phi_i = \bigl[
        \mathbf{x}_i,\;
        \mathbf{m}_i^{\mathrm{BC}},\;
        \mathbf{v}_i^{\mathrm{BC}},\;
        \mathbf{f}_i,\;
        p_i
    \bigr] \in \mathbb{R}^{d + 2N_F + N_F + N_P},
\end{equation}
where $\mathbf{x}_i \in \mathbb{R}^d$ are the node coordinates, $\mathbf{m}_i^{\mathrm{BC}} \in \{0,1\}^{N_F}$ is a per-field Dirichlet mask (1 if node $i$ carries a prescribed boundary value the field, 0 otherwise), $\mathbf{v}_i^{\mathrm{BC}} \in \mathbb{R}^{N_F}$ are the corresponding prescribed values, $\mathbf{f}_i \in \mathbb{R}^{N_F}$ is the per-field nodal forcing/source term, and finally $p_i \in \mathbb{R}^{N_P}$ are channels of the physical parameter inputs $\mu$ (e.g. diffusivity), consistent with the description in Section~\ref{sec:method}.

We compare against a pointwise MLP, a PointNet-style model~\citep{qi2017pointnet} with max-pooled global context (T-net omitted), a GraphSAGE-style GNN~\citep{hamilton2017inductive}, MeshGraphNet~\citep{pfaff2020learning} (adapted to a meshless $\epsilon$-ball graph with Euclidean edge features), DeepONet~\citep{lu2021learning} (branch net replaced by max-pool over node features to handle variable-$N$ point clouds), Transolver~\citep{wu2024transolver}, UPT~\citep{alkin2024universal}, linear attention~\citep{katharopoulos2020transformers}, and Galerkin attention~\citep{cao2021choose}.

\section{Broader impact statement}
\label{app:broader_impact}

This work develops data-efficient surrogates that learn local constitutive relations from a small number of high-fidelity simulations and transfer them across geometries, boundary conditions, and parameter regimes. The intended impact is in scientific and engineering settings where high-fidelity data is expensive to acquire, including subsurface flow, climate and weather subgrid modeling, biomedical mechanics, additive-manufacturing process simulation, and structural design. With the recent advent of scientific foundation models mandating large corpus of training data, efficient single shot learning can have a large environmental impact by making responsible use of datacenters and compute resources. By reducing the data and compute required to obtain a usable model, the method lowers the barrier to physics-based design in academic and small-industry settings that currently cannot afford large simulation campaigns. We imagine a longer-term impact where large language models are able to reason through these kinds of representations to link physical reasoning to contemporary computer vision.

The principal risk we see is misplaced trust in extrapolation. Our error decomposition (Theorem 1) holds only when test features lie in the training feature distribution. Practitioners deploying the method in safety-relevant workflows should verify feature coverage and validate against held-out high-fidelity simulations before relying on predictions. We do not see acute misuse risks of the kind associated with generative or surveillance technologies; the method provides a similar (albeit potentially enhanced) capability as conventional scientific simulators.

\end{document}